 \def\arXiv{1}
 \documentclass[11pt]{article}
\usepackage{enumerate}
\newcommand{\notarxiv}[1]{foo}
\newcommand{\arxiv}[1]{ba}
\ifdefined \arXiv
\renewcommand{\arxiv}[1]{#1}%
\renewcommand{\notarxiv}[1]{\ignorespaces}%
\else%
\renewcommand{\arxiv}[1]{\ignorespaces}%
\renewcommand{\notarxiv}[1]{#1}%
\fi%

\notarxiv{
	\usepackage{aliascnt}
	\newaliascnt{assumption}{theorem}
	\newtheorem{assumption}[assumption]{Assumption}
	\aliascntresetthe{assumption}
	\crefname{assumption}{assumption}{assumptions}
}

\arxiv{
	\usepackage{amsthm}
}

\usepackage{thm-restate}
\usepackage{xparse}
\usepackage{enumerate}
\usepackage{amssymb}
\usepackage{amsbsy}
\usepackage{amsfonts}
\usepackage{latexsym}
\usepackage{graphicx}
\usepackage{mathtools}
\usepackage{enumitem}
\usepackage{thmtools}
\usepackage{graphicx}
\usepackage{color}
\usepackage{bbm}
\usepackage{xifthen}
\usepackage{xspace}
\usepackage{titlesec}
\usepackage{etoolbox}
\usepackage{xfrac}
\usepackage{esvect}
\usepackage{nicefrac}
\usepackage{cases}
\usepackage{empheq}
\usepackage{booktabs}
\usepackage{tabto}
\usepackage{multirow}
\arxiv{
	\usepackage{caption}
}
\usepackage{bbm}
\usepackage{array}
\usepackage{xurl}
\usepackage{comment}
\usepackage{pifont}
\usepackage{tikz}

\usepackage{makecell}
\usepackage{arydshln}
\usepackage{soul}
\usepackage{adjustbox}
\usepackage{pdflscape}

\usepackage{xargs}

\arxiv{
	\usepackage[linesnumbered,ruled,vlined,boxed,algo2e]{algorithm2e} %
}

\notarxiv{
	\usepackage[utf8]{inputenc} %
	\usepackage[T1]{fontenc}    %
	\usepackage{hyperref}       %
	\usepackage{url}            %
}
\usepackage{microtype}      %
\notarxiv{

}

\arxiv{
	\usepackage[hidelinks]{hyperref}
}

\usepackage{nameref}
\hypersetup{
	colorlinks=true,
	linkcolor=blue!70!black,
	citecolor=blue!70!black,
	urlcolor=blue!70!black}

\arxiv{
	\usepackage{cleveref}
	\crefname{assumption}{Assumption}{Assumptions}

	\usepackage{crossreftools}
	\pdfstringdefDisableCommands{
		\let\Cref\crtCref
		\let\cref\crtcref
	}
}

\arxiv{
	\usepackage[top=1in, right=1in, left=1in, bottom=1in]{geometry}
	\usepackage[numbers,square]{natbib}

	\theoremstyle{plain}
	
	\newtheorem{theorem}{Theorem}
	\newtheorem{lemma}{Lemma}
	
	\newtheorem{proposition}{Proposition}
	\newtheorem{corollary}{Corollary}

	\theoremstyle{definition}
	
	\newtheorem{definition}{Definition}
	\newtheorem{assumption}{Assumption}

	\newtheorem*{example*}{Example}
}

\DeclarePairedDelimiter{\abs}{\lvert}{\rvert} %
\DeclarePairedDelimiter{\brk}{[}{]}
\DeclarePairedDelimiter{\crl}{\{}{\}}
\DeclarePairedDelimiter{\prn}{(}{)}

\DeclarePairedDelimiter{\norm}{\|}{\|}
\DeclarePairedDelimiter{\tri}{\langle}{\rangle}

\DeclarePairedDelimiter{\ceil}{\lceil}{\rceil}
\DeclarePairedDelimiter{\floor}{\lfloor}{\rfloor}

\DeclarePairedDelimiter{\brkprn}{[}{)}

\newcommand{\overle}[1]{\overset{#1}{\le}}

\NewDocumentCommand\Ex{s O{} m }{%
	\mathbb{E}%
	\begingroup
	\IfBooleanTF{#1}
	{\ExInn*{#3}}
	{\ExInn[#2]{#3}}%
	\endgroup
}

\DeclarePairedDelimiterX\ExInn[1]{[}{]}{%
	\activatebar
	#1%
}

\RenewDocumentCommand\Pr{sO{}r()}{%
	\mathbb{P}%
	\begingroup
	\IfBooleanTF{#1}
	{\PrInn*{#3}}
	{\PrInn[#2]{#3}}%
	\endgroup
}

\DeclarePairedDelimiterX\PrInn[1](){%
	\activatebar
	#1%
}

\newcommand{\activatebar}{%
	\begingroup\lccode`~=`|
	\lowercase{\endgroup\def~}{\;\delimsize\vert\;}%
	\mathcode`|=\string"8000
}

\newcommand\numberthis{\addtocounter{equation}{1}\tag{\theequation}}

\newcommand{\R}{\mathbb{R}} %
\newcommand{\N}{\mathbb{N}} %
\newcommand{\E}{\mathbb{E}} %
\renewcommand{\P}{\mathbb{P}}	%
\newcommand{\I}{\mathbb{I}} %

\usepackage{xcolor}

 \DeclareMathOperator*{\argmax}{arg\,max}
 \DeclareMathOperator*{\argmin}{arg\,min}

\newcommand{\X}{\mathcal{X}}
\newcommand{\Y}{\mathcal{Y}}
\newcommand{\F}{\mathcal{F}}

\providecommand{\abs}{\mathop{\rm abs}}

\providecommand{\minimize}{\mathop{\rm minimize}}

\newcommand{\defeq}{\coloneqq}

\newcommand{\grad}{\nabla}

\renewcommand{\O}[1]{O\left( #1 \right)}
\newcommand{\Otil}[1]{\widetilde{O}\left( #1 \right)}

\newcommand{\Om}[1]{\Omega\left( #1 \right)}

\newcommand{\OT}[1]{\varTheta\left( #1 \right)}

\newcommand{\fracs}[2]{\prn*{#1/#2}}

\newcommand{\g}[1][t]{g_{#1}}

\NewDocumentCommand{\llft}{ O{\delta} O{t} }{\lambda_{#2}(#1)}



\newcommand{\opt}{_\star}
\newcommand{\xopt}{x\opt}

\newcommand{\Var}{\mathrm{Var}}

\makeatletter
\def\IfEmptyTF#1{%
	\if\relax\detokenize{#1}\relax
	\expandafter\@firstoftwo
	\else
	\expandafter\@secondoftwo
	\fi}
\makeatother

\SetKwInput{Input}{Input}                %
\SetKwInput{Output}{Output}              %
\SetKwInput{Parameters}{Parameters}
\SetKwProg{Function}{function}{}{}
\SetKwComment{Comment}{$\triangleright$\ }{}
\SetKwRepeat{Do}{do}{while}
\SetKwBlock{Repeat}{Repeat}{}
\SetCommentSty{mycommfont}
\makeatletter
\newcommand\Block[2]{%
	#1%
	\algocf@group{#2}%
}
\makeatother

\newcommand{\gradientOracle}[1]{\IfEmptyTF{#1}{\mathcal{G}}{\mathcal{G}(#1)}}

\newcommand{\Proj}[2]{\mathbf{\mathrm{Clip}}_{#1}(#2)}

\newcommand{\barF}[1]{\bar{F}\prn*{x_{#1}}}

\newcommand{\lnp}{\ln_+}

\newcommand{\citeg}[1]{\citep[e.g.,][]{#1}}  %

\newcommand{\indic}[1]{\mathbbm{1}_{\!\left\{#1\right\}}} %

\newcommand{\fL}{\hat{L}}
\newcommand{\fF}{\hat{F}}

\newcommand{\lt}[2]{h\prn*{#1;#2}}

\newcommand{\mt}[2]{m\prn*{#1;#2}}
\newcommand{\ml}[3]{m_{#2}\prn*{#1;#3}}

\newcommand{\qt}[2]{q\prn*{#1;#2}}

\newcommand{\ft}[2]{f\prn*{#1;#2}}
\newcommand{\fl}[3]{f_{#2}\prn*{#1;#3}}

\newcommand{\Ft}[2]{F_{#2}\prn*{#1}}
\newcommand{\Fo}[1]{F\prn*{#1}}

\newcommand{\fFo}[1]{\fF\prn*{#1}}

\newcommand{\ue}[2]{u\prn*{#1,#2}}
\newcommand{\ve}[2]{v\prn*{#1,#2}}
\newcommand{\at}[2]{a\prn*{#1,#2}}
\newcommand{\te}[2]{\tau\prn*{#1,#2}}

\renewcommand{\S}{\mathbb{S}}
\newcommand{\PP}{\mathcal{P}}
\newcommand{\D}{\mathbb{D}}
\newcommand{\Hc}{\mathcal{H}}

\newcommand{\grid}[3][2]{\text{Grid}_{#1} \prn*{ #2, #3 } }
\newcommand{\Grid}{\mathcal{G}}
\newcommand{\C}{\mathcal{C}}

\newcommand{\lell}{\underline{\ell}}
\newcommand{\uell}{\overline{\ell}}
\newcommand{\lrho}{\underline{\rho}}
\newcommand{\urho}{\overline{\rho}}

\newcommand{\Dopt}{\overline{D}\opt}
\newcommand{\Lopt}{\overline{L}\opt}

\newcommand{\Fbar}[1][T]{\bar{F}_{T} }

\NewDocumentCommand{\norma}{O{\alpha} m}{\norm{#2}_{#1}}
\NewDocumentCommand{\normb}{O{\beta} m}{\norm{#2}_{#1}}
\NewDocumentCommand{\normo}{O{\alpha} O{\beta} m}{ \norm{#3}_{#1,#2} }
\NewDocumentCommand{\normbo}{O{\beta} O{\text{abs}} m}{ \norm{#3}_{#1,#2} }

\NewDocumentCommand{\Q}{m m}{ Q_{#2}\prn*{#1} }
\NewDocumentCommand{\QQ}{O{m}}{ \mathbb{Q}_{#1} }

\newcommand{\normas}[1]{\norm*{#1}_{\alpha *}}
\newcommand{\normbs}[1]{\norm*{#1}_{\beta *}}

\newcommand{\deltah}{\hat{\delta}}

\newcommandx{\al}[2][1=x, 2=b]{\text{binary-ceil}\prn*{#1; #2}}

\newcommandx{\alg}{\mathsf{alg}}
\renewcommand{\AA}{\mathcal{A}}
\renewcommand{\I}{\mathcal{I}}
\NewDocumentCommand{\IsLip}{O{\lell} O{\uell} O{\norm{\cdot}}}{ \I^{#2,#3}_{\text{SM-Lip}} }
\NewDocumentCommand{\ILip}{O{\lell} O{\uell} O{\norm{\cdot}}}{\I^{#2,#3}_{\text{Lip}}}
\NewDocumentCommand{\IDist}{O{\lrho} O{\urho} O{\norma{\cdot}}}{\I^{#2,#3}_{\text{Dist}}}
\NewDocumentCommand{\IDistf}{O{\lrho} O{\urho} O{\norma{\cdot}}}{\hat{\I}^{#2,#3}_{\text{Dist}}}

\NewDocumentCommand{\ISOML}{}{\I_{\text{SOML}}}
\NewDocumentCommand{\ISOMLl}{O{L} O{D}}{\I^{#1,#2}_{\text{SOML-Lip}}}
\NewDocumentCommand{\ISOMLs}{O{L} O{D}}{\I^{#1,#2}_{\text{SOML-sm-Lip}}}

\NewDocumentCommand{\Icont}{}{\I_{\text{cont}}}
\NewDocumentCommand{\Idiff}{}{\I_{\text{diff}}}

\NewDocumentCommand{\Err}{O{N} O{\delta} m m}{\mathsf{Err}^{#2}_{#1}\prn*{#3, #4}}

\newcommand{\M}{\mathcal{M}}
\NewDocumentCommand{\MLip}{O{\lell} O{\uell} O{\lrho} O{\urho}}{\M^{#1,#2,#3,#4}_{\text{Lip}}}
\NewDocumentCommand{\MsLip}{O{\lell} O{\uell} O{\lrho} O{\urho}}{\M^{#1,#2,#3,#4}_{\text{sm-Lip}}}

\NewDocumentCommand{\PoA}{O{N} O{\delta} O{\alg} m m}{\mathsf{PoA}_{#1}^{#2}\prn*{ #3, #4 ; #5 }}

\newcommand{\Crefn}[1]{\Cref{#1}}%

\newcommand{\Creflem}[1]{\texorpdfstring{\Cref{#1}}{Lemma~\ref{#1}}}
\newcommand{\Crefpro}[1]{\texorpdfstring{\Cref{#1}}{Proposition~\ref{#1}}}
\newcommand{\Crefthm}[1]{\texorpdfstring{\Cref{#1}}{Theorem~\ref{#1}}}

\newcommand{\iid}{\overset{\text{iid}}{\sim}}

\newcommand{\clarke}{\partial_c}

\newcommand{\titleCrefi}[2]{%
	\texorpdfstring{\hyperref[#2]{\labelcref*{#1}.\ref*{#2}}}%
	{\ref*{#1}.\ref*{#2}}
}

\makeatletter

\def\mynameref#1#2{%
	\begingroup
	\edef\@mytxt{#2}%
	\edef\@mytst{\expandafter\@thirdoffive\@mytxt}%
	\ifx\@mytst\empty\else
	\space(\nameref{#1})\fi
	\endgroup
}

\makeatother

\makeatletter
\long\def\@makecaption#1#2{
  \vskip 0.8ex
  \setbox\@tempboxa\hbox{\small {#1:} #2}
  \parindent 1.5em  %
  \dimen0=\hsize
  \advance\dimen0 by -3em
  \ifdim \wd\@tempboxa >\dimen0
  \hbox to \hsize{
    \parindent 0em
    \hfil
    \parbox{\dimen0}{\def\baselinestretch{0.96}\small
      {#1:} #2
    }
    \hfil}
  \else \hbox to \hsize{\hfil \box\@tempboxa \hfil}
  \fi
}
\makeatother

\newcommand{\yair}[1]{{\bf \color{blue} Yair: #1}}
\newcommand{\ollie}[1]{{\bf \color{green!50!black} Ollie: #1}}
\newcommand{\maor}[1]{{\bf \color{cyan} Maor: #1}}
\newcommand{\itai}[1]{{\bf \color{orange} Itai: #1}}
\newcommand{\TBD}[1]{{\bf \color{red} TBD: #1}}
\newcommand{\TODO}[1]{{\bf \color{purple} TODO: #1}}

\newcommand{\edit}[1]{{\bf \color{red} Proposed edit: #1}}

\renewcommand{\yair}[1]{\ignorespaces}
\renewcommand{\ollie}[1]{\ignorespaces}
\renewcommand{\maor}[1]{\ignorespaces}
\renewcommand{\itai}[1]{\ignorespaces}
\renewcommand{\TBD}[1]{\ignorespaces}
\renewcommand{\TODO}[1]{\ignorespaces}

\renewcommand{\edit}[1]{\ignorespaces} \notarxiv{\usepackage{times}}

\title{Clipping the Price of Adaptivity at the Tail}

\notarxiv{
	\coltauthor{%
		\Name{Itai Kreisler} \Email{kreisler@mail.tau.ac.il}\\
		\Name{Yair Carmon} \Email{ycarmon@tauex.tau.ac.il} \\
		\Name{Oliver Hinder} \Email{ohinder@pitt.edu}
	}
}
\arxiv{
	\author{Itai Kreisler\thanks{Tel Aviv University,
			\href{mailto:kreisler@mail.tau.ac.il}{\texttt{kreisler@mail.tau.ac.il}},
			\href{mailto:ycarmon@tauex.tau.ac.il}{\texttt{ycarmon@tauex.tau.ac.il}.}
		}
		~~~~~~Yair Carmon\footnotemark[1]~~~~~~Oliver Hinder\thanks{University of Pittsburgh, \href{mailto:ohinder@pitt.edu}{\texttt{ohinder@pitt.edu}}}}

	\date{}
}

\begin{document}

\maketitle

\begin{abstract}%
	Adaptive stochastic convex optimization (SCO) methods face a fundamental ``price of adaptivity'' barrier: under the standard set of assumptions, they cannot efficiently adapt to large uncertainty in both the initial distance to optimality and the Lipschitz constant.
	We circumvent this barrier by requiring a small amount of additional structure common to many learning problems.
	Specifically, we assume that the objective decomposes into a model and a loss function, enabling us to intervene by modifying the model's output before it passes to the loss function.
	Under this assumption, we design a method that clips the learned model output in tail events where it deviates too much from the output of a fixed reference model.
	Our method matches the optimal bounds for known-parameter SCO up to logarithmic factors in the uncertainty in the distance and Lipschitz parameters, thus efficiently adapting to large uncertainty in both.
\end{abstract}

\notarxiv{
}

\section{Introduction}
This paper studies machine learning objectives composed of a model and a loss function.
Specifically, for a model $m$ and a loss function $h$, the learning problem can be written as
\begin{flalign}\label{eq:ML-optimization}
	\minimize_{x} \Fo{x}, \quad \text{where} \quad   \Fo{x} \defeq \E_{s\sim\PP} \brk*{ \lt{ \mt{x}{s}}{ s } }.
\end{flalign}

Such model--loss decompositions arise broadly in practice, e.g., neural networks with various loss functions.
However, optimization theory typically analyzes the more general problem of minimizing a stochastic function $f$ without assuming this internal structure:
\begin{align}\label{eq:general-stochastic-convex}
	\minimize_{x} \Fo{x}, \quad \text{where} \quad  \Fo{x} \defeq \E_{s\sim\PP} \brk*{ \ft{x}{s} }.
\end{align}

Stochastic convex optimization is a fundamental problem in machine learning and serves as the foundation for classical stochastic gradient methods.
While classical methods can attain the optimal optimality gap, they typically hinge on knowing the Lipschitz parameter $L\opt$ and the distance $D\opt$ from the initial point to the optimum.
\emph{Parameter-free} methods \citeg{orabona2016coin, cutkosky2018black, bhaskara2020online, mhammedi2020lipschitz, jacobsen2022parameter, carmon2022making, ivgi2023dog, kreisler24accelerated} relax the requirement for precise knowledge of the problem parameters, aiming to achieve a near-optimal optimality gap using only loose estimates.
Lower bounds \citep{carmon2024price, attia2024free, khaled2024tuning} for parameter-free optimization show that parameter-free methods must pay a significant multiplicative price due to uncertainty in $L\opt$ and $D\opt$.
For example, \citet{carmon2024price} show that if $\rho$ and $\ell$ are the multiplicative uncertainties in the distance to optimality and Lipschitz constant, respectively, then there is an $\Omega( \min\crl*{\ell,\rho}\sqrt{\ln\fracs{1}{\delta}}/\sqrt{T})$ additional multiplicative cost from not knowing problem parameters.
This shows that it is difficult to adapt to high uncertainty in both the Lipschitz constant and distance to optimality.
However, this lower bound applies to the general stochastic convex structure given in \Cref{eq:general-stochastic-convex} but not the model--loss structure of \Cref{eq:ML-optimization}.

\subsection*{Our contributions}

We identify that the aforementioned lower bound for parameter-free methods is intimately tied to tail events: rare samples for which the optimized function behaves very differently from the bulk of the distribution.
Such rare samples may go unobserved by the optimization algorithm, making it impossible to reliably characterize the tail behavior of the stochastic noise.

However, under the model--loss structure of \Cref{eq:ML-optimization}, it is possible to detect and mitigate these rare events, circumventing this lower bound.
In particular, under this model--loss structure, our key idea is to clip the output of the model at \emph{inference-time}.
This clipping occurs on rare events in which the model provides an extremely large outlier prediction.
In these cases, clipping prevents us from paying a high price for a wildly incorrect prediction.
We present two methods that use this idea: one that prioritizes computational efficiency and the other that prioritizes sample efficiency.
Both methods build on \citet{lawrence2025sample}, which focused on the standard parameter-free optimization paradigm, i.e., \Cref{eq:general-stochastic-convex}.

In the computational setting, we propose a method that performs grid search over hyperparameters, generates candidate weights using a standard optimization algorithm such as SGD, and then selects one of the resulting weights paired with an associated clipping parameter.
We show that this method, with probability at least $1-\delta$, achieves an optimality gap of $\widetilde{O}(L\opt D\opt\sqrt{\ln\fracs{1}{\delta}}/\sqrt{T})$ where $T$ is the number of gradient and function evaluations made by the method, and $\widetilde{O}(\cdot)$ hides poly-logarithmic factors.
Notably, our method only incurs a logarithmic cost in uncertainty in the problem parameters compared with the polynomial dependence exhibited in the lower bounds.
This optimality gap beats the lower bounds proved in \citet{carmon2024price, attia2024free, khaled2024tuning}, which scale linearly with $\min\{ \ell, \rho \}$.

In the sample complexity setting, we propose a method that carefully solves a regularized optimization problem and then applies inference-time clipping to the resulting learned model.
This gives a suboptimality gap of $O(L\opt D\opt \sqrt{\ln\fracs{1}{\delta}}/\sqrt{N})$ for the Lipschitz case, where $N$ is the number of samples made available to the algorithm.
This matches lower bounds for parameter-known stochastic convex optimization \citep{carmon2024price}.

\subsection*{Paper outline}

\Cref{sec:related-work} surveys related work.
\Cref{sec: notation} defines the notation used throughout the paper, as well as the instance classes of functions considered in this work.
In \Cref{sec: comp algorithm} we propose a parameter-free method that achieves a near-optimal optimality gap in the computational complexity setting.
Specifically, in \Cref{sec: model selection} we introduce a model selection algorithm, and in \Cref{sec: conv rate} we show how to use this algorithm to derive a parameter-free optimization method.
In \Cref{sec: sample algorithm} we propose a parameter-free method in the sample complexity setting that achieves the optimal optimality gap, without requiring any prior lower or upper bounds on the problem parameters $L\opt$ and $D\opt$.
Finally, \Cref{sec: lower bounds} discusses which lower bounds continue to hold under the model--loss decomposition assumption defined in \Cref{eq:ML-optimization}.

\section{Related Work}\label{sec:related-work}

\paragraph{Parameter-free methods}
Parameter-free methods aim to achieve a near-optimal optimality gap without requiring precise knowledge of the problem parameters, historically with emphasis on the initial distance to the optimum.
Work in this field originates in the setting of online convex optimization (OCO) \citep{luo2015achieving,orabona2016coin,cutkosky2018black,mhammedi2020lipschitz,bhaskara2020online,jacobsen2022parameter,zhang2022pde}.
This line of work was subsequently expanded to design parameter-free methods, particularly for the setting of stochastic convex optimization (SCO) \citeg{orabona2014simultaneous,orabona2017training,chen2022better,carmon2022making,ivgi2023dog,kreisler24accelerated,liu2025dual}.

\paragraph{Lower bounds}
In non-smooth stochastic convex optimization with known Lipschitz constant, $L\opt$, and initial distance to optimality, $D\opt$, classical results show that the expected loss is $\Om{L\opt D\opt/\sqrt{T}}$ \citep{nemirovski1983problem,duchi2018introductory}.
For the high-probability case, \citet{carmon2024price} proves the lower bound of $\Om{L\opt D\opt \sqrt{\ln\fracs{1}{\delta}}/\sqrt{T}}$. These lower bounds have matching upper bounds \citeg{orabona2021modern}, and consequently define the optimal optimality gaps for non-smooth stochastic convex optimization.

Recently, \citet{carmon2024price} proves that parameter-free methods incur a penalty due to uncertainty in the Lipschitz parameter $L\opt$ and the distance $D\opt$ from the initial point to the optimum.
Specifically, they assume that while $L\opt$ and $D\opt$ are unknown, they are bounded such that $L\opt\in\brk*{1,\ell}$ and $D\opt\in\brk*{1,\rho}$ for some known values $\ell$ and $\rho$.
They prove that if stochastic gradients are almost surely bounded by $L\opt$, the optimality gap necessarily suffers a multiplicative degradation of at least $\Om{ \min\crl*{\ell,\rho}\sqrt{\ln\fracs{1}{\delta}}/\sqrt{T}}$ relative to the optimal rate.
Furthermore, they prove that in the second-moment setting, in which $L\opt^2$ bounds the expectation of the stochastic gradient norm squared, the multiplicative penalty is at least $\Om{\min\crl*{\ell,\rho}}$.
This implies that, in heavy-tailed settings, achieving a near-optimal optimality gap is impossible without prior knowledge of at least one of the problem parameters.

\citet{attia2024free} and \citet{khaled2024tuning} also provide lower bounds for parameter-free optimization.
\citet{attia2024free} derives comparable lower bounds to \citet{carmon2024price} in the Lipschitz case.
\citet{khaled2024tuning} provides a slightly weaker lower bound than \citet{attia2024free}, demonstrating that the multiplicative price incurred due to parameter uncertainty cannot be bounded by any polylogarithmic term.

Critically, all these lower bounds focus on the standard stochastic optimization setting, i.e., \Cref{eq:general-stochastic-convex}, allowing us to beat them by considering the model--loss structure, i.e., \Cref{eq:ML-optimization}. We note that \emph{some} of the lower bounds presented in \citet{carmon2024price} do transfer to \Cref{eq:ML-optimization}; we discuss this in more detail in \Cref{sec: lower bounds} and its supplement, \Cref{sec: Formal lower bounds}.

\paragraph{Reliable Model Selection}
\citet{lawrence2025sample} demonstrates that the standard holdout model selection approach can fail catastrophically in SCO.
To address this, they introduce a general framework for reliable model selection.
They also propose an optimization method for the sample complexity setting.
However, both methods require a good estimate of the Lipschitz constant (i.e., within a factor of $\sqrt{T}$ of the true value).
In this paper, we build upon and extend many of the results and proof techniques introduced by \citet{lawrence2025sample}, eliminating the need for a good estimate of the Lipschitz constant.

\paragraph{Improper learning}
In optimization and machine learning theory, one typically fixes a hypothesis class $\Hc$ (e.g., linear classifiers) as a representation for the target concept class.
A learning algorithm is said to be \emph{proper} if it is required to output a hypothesis from $\Hc$.
In contrast, an algorithm is called \emph{improper} if it is allowed to output any predictor, possibly outside $\Hc$, as long as the predictor attains low prediction error under the data distribution.
This relaxation often substantially enlarges the space of candidate predictors and consequently can yield learning algorithms that succeed in settings where proper learning is statistically or computationally infeasible \citeg{montasser2019vc,hanneke2016optimal}.

The algorithms we present in this work can be viewed as improper, as they output a weight $x$ and a clipping parameter $c$ for model $\ml{x}{c}{\cdot}$, rather than just outputting a weight $x$ for the model $\mt{x}{\cdot}$.
Unlike many works in this field that output an ensemble of predictors aggregated via majority vote, our approach outputs a single model, using the initial hypothesis solely to constrain the predictions.

A particularly relevant piece of work in this area is \citet{vovk2006line}, who provide an online learning method that exploits a loss--model structure in the least squares setup.
However, that work assumes known bounds on the labels.

\section{Notation}
\label{sec: notation}

We include $0$ in the set of natural numbers $\N$.
For an integer $K\in\N$ we define $\brk*{K} \defeq \crl*{ 1, \cdots, K }$.
We denote $\lnp\prn*{x} = \ln\prn*{e + x}$.
Let $\X,\Y$ denote closed convex sets, let $\S$ denote a sample space, and let $\PP$ denote a distribution over $\S$.
We denote the number of samples by $N$, and the number of model and gradient evaluations by $T$.
Let $\norma{\cdot}$ and $\normb{\cdot}$ be norms on $\X$ and $\Y$, respectively.
Let $\normas{\cdot}$ and $\normbs{\cdot}$ be the dual norms of $\norma{\cdot}$ and $\normb{\cdot}$, respectively.
The corresponding operator norm, induced by the norms $\norma{\cdot}$ and $\normb{\cdot}$, is defined as $\normo{A} \defeq \sup\crl*{ \normb{Ax} ~|~ \norma{x}\le1 }$.
We denote by $\Q{\D}{p}\defeq\min\crl*{ z\in\R : \P_{s\sim\D}\brk*{ s \le z } \ge p }$ the $p$-th quantile of distribution $\D$, where $\D$ is a distribution over $\R$; this definition mimics that of previous work \citep[e.g.,][Ch. 3]{borovkov1999probability}.

Throughout, we assume that the function $f:\X \times \S \rightarrow \R$ admits a decomposition into a model and a loss function, i.e., $\ft{x}{s}\defeq \lt{ \mt{x}{s} }{ s }$, where $m:\X \times \S \rightarrow \Y$ and $h:\Y \times \S \rightarrow \R$.
For convenience, we treat the functions $m,h$ and $f$ as differentiable.
We denote the class of all differentiable functions as
$\Idiff \defeq \crl*{ q ~|~ \forall s\in \S: \qt{\cdot}{s} \text{ is differentiable}}$.
In \cref{sec: remove differentiable} we discuss how to extend our results to non-differentiable functions.
We define $\Fo{x} \defeq \E_{s\sim\PP} \brk*{ \ft{x}{s} }$.
We denote by $x_0$ the point at which the optimization algorithms are initialized.
We define a clipping function into a domain $\mathcal{D}\subseteq\Y$ as a mapping $\Proj{\mathcal{D}}{\cdot}:\Y\to\mathcal{D}$ such that $\Proj{\mathcal{D}}{z}=z$ for all $z\in\mathcal{D}$.
We allow $\Proj{\mathcal{D}}{\cdot}$ to be any such mapping; any choice suffices for our theoretical analysis.
We define the clipped model as $\ml{x}{c}{s} \defeq \Proj{\crl*{u\in\Y : \normb{u - \mt{x_0}{s}} \le c }}{ \mt{x}{s} }$, where $c\in\R$ is the clipping parameter.
Subsequently, we define the clipped function as $\fl{x}{c}{s}\defeq \lt{ \ml{x}{c}{s} }{ s }$ and $\Ft{x}{c} \defeq \E_{s\sim\PP} \brk*{ \fl{x}{c}{s} }$.

We define $\al[x][b] \defeq \min\crl*{ 2^{n} b ~|~ n\in\N \text{ and } 2^{n} b \ge x \indic{b\neq0}}$, i.e., rounding $x\neq0$ up to the nearest power-of-two multiple of $b$.
For any $p>1$ and $\bar{b}\ge\underline{b}>0$, we define $\grid[p]{\underline{b}}{\bar{b}}\defeq\crl*{ \underline{b} \cdot p^n }_{n=0}^{\floor{\log_p\prn*{\bar{b}/\underline{b}}}}\cup\crl*{\bar{b}}$, i.e., all the powers of $p$ between the bounds $\underline{b}$ and $\bar{b}$.

In addition, we assume throughout that $F$ attains its minimum, i.e., there exists $\xopt\in\X$ such that $\Fo{\xopt}=\inf_{x\in\X}\Fo{x}$. We denote by $\xopt$ the optimal point closest to the initialization $x_0$, i.e., $\xopt \defeq \argmin_{x \in \text{optima of }F}\norma{x_0 - x}$.
We define the unknown initial distance to the optimum as $D\opt \defeq\norma{x_0 - \xopt}$, and denote by $\Dopt \defeq \max\prn*{ D\opt, \lrho }$ the maximum of $D\opt$ and a known lower bound $\lrho$ on $D\opt$.
We define the instance classes of convex stochastic optimization problems whose optimal solution lies within distance at most $\urho$ of the initialization as
\begin{align*}
	\IDist &\defeq \crl*{q\in\Idiff ~|~ \norma{x_0 - \xopt} \le \urho \text{ and } \E_{s\sim \PP} \qt{\cdot}{s} \text{ is convex}},\\
	\IDistf &\defeq \crl*{q\in\Idiff ~|~ \norma{x_0 - \xopt} \le \urho \text{ and } \forall s\in \S: \qt{\cdot}{s} \text{ is convex}}.
\end{align*}

Similarly, we define the unknown Lipschitz constant in the context of the Lipschitz case as $L\opt \defeq \sup_{x\in\X ; s\in \S} \normo{\grad_x \mt{x}{s}}$,
and in the context of the second-moment-Lipschitz case as $L\opt \defeq \sup_{x\in\X} \sqrt{\E_{s\sim \PP} \normo{\grad_x \mt{x}{s}}^2}$.
We denote $\Lopt \defeq \max\prn*{ L\opt, \lell }$ as the maximum of $L\opt$ and a known lower bound $\lell$ on $L\opt$.
We then define the instance classes of Lipschitz and second-moment-Lipschitz problems with Lipschitz constant at most $\uell$ as
\begin{align*}
	\ILip &\defeq \crl*{q\in\Idiff ~|~ \forall x\in\X; s\in \S: \norm{\grad_x \qt{x}{s}} \le \uell },\\
	\IsLip &\defeq \crl*{q\in\Idiff ~|~ \forall x\in\X: \E_{s\sim \PP} \norm{\grad_x \qt{x}{s}}^2  \le \uell^2 }.
\end{align*}
We note that any Lipschitz function is also a second-moment-Lipschitz function, i.e., $\ILip \subset \IsLip$, and that $\IsLip$ also includes problems with heavy-tailed gradient noise.

\section{A computationally efficient parameter-free method that exploits model--loss structure}
\label{sec: comp algorithm}

In this section, we propose a parameter-free method that achieves a near-optimal optimality gap.
To this end, our method operates under the model--loss decomposition described in \Cref{eq:ML-optimization} and introduces clipping at the model output level.

Our parameter-free method combines grid search, a known optimization method, and a model selection procedure.
By performing grid search over possible hyperparameters and applying a known optimization method, such as SGD, at each grid point, we generate a collection of candidate weights, one of which attains a small optimality gap.
We then introduce a model selection method that, given this collection of candidates, selects a candidate paired with a clipping parameter such that the selected pair achieves a loss close to the minimum loss among all candidates.

\subsection{The prior work that we leverage}

Our model selection algorithm uses the model selection framework proposed by \citet{lawrence2025sample}.
Their algorithm takes as input a reference model $x_0$, a list of candidate models $x_1,\dots,x_K$, and corresponding confidence interval widths $\tau_1,\dots,\tau_K$.
A key condition is the requirement of \Cref{lem: reliable selection results}, which states that each $\tau_k$ bounds the deviation between the validation loss and the population loss of $x_k$, relative to the reference model $x_0$.
From the list of models $x_0,\dots,x_K$, their algorithm selects one that balances a low validation loss with a low risk, where the latter is reflected by a small $\tau$.
The selected model enjoys a strong guarantee on its performance relative to the other candidates $x_0,\dots,x_K$;
this guarantee is stated explicitly in \Cref{lem: reliable selection results}.
Consequently, obtaining confidence widths $\tau_1,\dots,\tau_K$ that are both valid and sufficiently small is crucial for the effectiveness of the reliable model selection method of \citet{lawrence2025sample}.
For completeness, we present their reliable model selection method in \Cref{sec: reliable model selection} as \Crefn{alg: reliable model selection}.

\begin{lemma}[Lemma 1 of \citet{lawrence2025sample}]
\label{lem: reliable selection results}
Let $\deltah\in\prn*{0,1}$.
Let $K$ be the number of models, and $N$ be the validation sample set size.
Let $\tau_{1}, \dots, \tau_{K}$ be non-negative scalars such that with probability $1-\deltah$ for all $k\in\brk{K}$ we have
$\abs*{ \Fo{x_k} - \Fo{x_0} - \frac{1}{N} \sum_{i=1}^{N} \prn*{ \ft{x_k}{s_i} -  \ft{x_0}{s_i} } } \le \tau_{k}$, and let $\tau_{0}=0$.
Then, with probability $1-\deltah$, $\Fo{x_\text{reliable}} \le \min_{k\in\brk*{K}\cup\crl*{0}} \brk*{ \Fo{x_k} + 2 \tau_{k} }$, where $x_\text{reliable}$ is the result of \Cref{alg: reliable model selection}.
\end{lemma}

\subsection{Our model selection algorithm}
\label{sec: model selection}
We present a model selection method \Cref{alg: clipping around initial}, and show in \Cref{thm: clipped model subopt} and corollaries in \Cref{sec: conv rate} that combining a grid search with our algorithm yields a result that, with high probability, has a near-optimal suboptimality.

The purpose of \Cref{alg: clipping around initial} is to select a weight $x$ paired with a clipping parameter $c$ such that the loss $\Ft{x}{c}$ is small relative to the losses of all candidate weights.
The algorithm receives a set of weight candidates $\crl*{x_g}_{g\in\Grid}$, possibly generated using some optimization method, such as grid search over SGD step sizes.
For each candidate $x_g$, it considers a set of clipping parameters $\C_g$, where $\C_g$ is defined as the $n$ largest elements of the set $\crl*{\normb{\mt{x_g}{s_i} - \mt{x_0}{s_i}}}_{i=1}^{N}$.
For each pair of weight and clipping parameter, the algorithm computes a confidence interval width $\tau$.
Finally, \Cref{alg: clipping around initial} applies \Crefn{alg: reliable model selection} to select the pair $(x,c)$ that minimizes the validation loss augmented by the risk term $\tau$.
The resulting weight and clipping parameter are guaranteed to achieve a loss close to that of the best candidate weight.

The key ingredient of the algorithm is the introduction of clipping at the model output at inference-time.
This clipping enables the use of \Crefn{alg: reliable model selection} to choose a weight and clipping parameter with sufficiently small suboptimality.
In contrast, attempting to achieve the same suboptimality by using grid search alone, without clipping the model output, is doomed to fail, as it contradicts \citet[Theorem 3]{carmon2024price}.

Introducing clipping at the model output allows us to characterize a confidence interval width $\tau$.
For each pair $(x,c)$, we compute the empirical mean of the difference between the clipped loss at $(x,c)$ and the loss at the initial weight $x_0$, and denote it by $\ue{x}{c} \defeq \frac{1}{N} \sum_{j=1}^{N} \prn*{ \fl{x}{c}{s_j} - \ft{x_0}{s_j} }$.
Similarly, we compute the empirical variance $\ve{x}{c} \defeq \frac{1}{N - 1} \sum_{j=1}^{N} \prn*{ \prn*{\fl{x}{c}{s_j} - \ft{x_0}{s_j} }- \ue{x}{c} }^2$.
To account for the union bound over candidate pairs, we define $\at{x}{c} \defeq \abs{\Grid} \log_2^2 \prn*{ \frac{2c}{\lell \norma{x-x_0} } }$ if $x \neq x_0$, and $\at{x}{c} \defeq \abs{\Grid}$ for $x = x_0$.
Finally, the confidence interval width is defined as
\begin{align}
	\label{eq: tau def}
	\te{x}{c} \defeq \sqrt{ \frac{ 2 \ve{x}{c} \cdot \ln\prn*{16  \at{x}{c} / \delta} }{ N } } + c \frac{14 \ln\prn*{16  \at{x}{c} / \delta} }{ 3 \prn*{ N - 1 } }.
\end{align}

To ensure that the set $\C_g$ contains a sufficiently small clipping parameter--specifically, one no larger than
$\al[\sqrt{\frac{N}{\ln(4/\delta)}} L\opt \norma{x_g - x_0}][\lell \norma{x_g - x_0}]$, while keeping $\C_g$ small, we set
$n = \ceil*{4 \ln(4/\delta)}$.

		\begin{algorithm2e}[hbt]
			\caption{Model selection with model output clipping}
			\label{alg: clipping around initial}
			\LinesNumbered
			\DontPrintSemicolon
			\Input{
				Initialization $x_0$, candidates $\crl*{x_g}_{g\in\Grid}$, number of validation set samples $N\ge2$, samples
				$s_1,\dots,s_N\iid\PP$, a lower bound $\lell$ on the Lipschitz constant, and a confidence level $\delta$.
			}

			Set $n=\ceil*{ 4 \ln \prn*{ 4/\delta} }$.\\
			\For{ $g\in \Grid$ }{
				Let $\C_g$ be the $n$ largest elements in $\crl*{\al[\normb{\mt{x_g}{s_i} - \mt{x_0}{s_i}}][\lell \norma{x_g - x_0}]}_{i=1}^{N}$.
				\\
				\For{ $c_{g,i} \in \C_g$ }{
					Define the clipped model: $\ml{x_g}{c_{g,i}}{s} \defeq \Proj{\crl*{u : \normb{u - \mt{x_0}{s}} \le c_{g,i} }}{ \mt{x_g}{s} }$.\;
					Set $\tau_{g,i} = \te{x_g}{c_{g,i}}$, where $\te{\cdot}{\cdot}$ is as defined in \Cref{eq: tau def}.
				}
			}
			Run \Crefn{alg: reliable model selection} with sample set of $\crl*{s_i}_{i\in\brk*{N}}$, with the initial weight of $x_0$ as the reference model, $\crl*{\prn*{ x_g, c_{g,i} }}_{\prn*{g,i}\in \Grid \times \brk*{n}}$ as the additional candidates, and $\crl*{\tau_{g,i}}_{\prn*{g,i}\in \Grid \times \brk*{n}}$ as the confidence intervals.\;
			\Return the resulting $x,c$ of the model selection algorithm.
		\end{algorithm2e}

\Cref{alg: clipping around initial} invokes \Crefn{alg: reliable model selection}.
Therefore, we apply \Cref{lem: reliable selection results} to guarantee that the result of \Cref{alg: clipping around initial} achieves the desired small loss.
To this end, we first show in the lemma below that the requirement of \Cref{lem: reliable selection results} is satisfied.
We present the proof of this lemma in \Cref{sec: assumption hold for alg}.
\begin{restatable}{lemma}{reslem}
	\label{lem: condition hold for algorithm}
	Let $N\ge2$ be the number of validation samples.
	For some $n\in\N$, let $\crl*{\prn*{x_{g_i},c_i}}_{i\in\brk*{n}}$ be a set of candidate pairs such that, for all $i\in\brk*{n}$ we have $x_{g_i}\in\crl*{x_g}_{g\in\Grid}$ and $c_i \in \crl*{ 2^k\lell\norma{x_{g_i}-x_0} ~|~ k\in\N }$.
	For each $i\in\brk*{n}$, define $\tau_{i}\defeq \te{x_{g_i}}{c_i}$.
	The requirement of \Cref{lem: reliable selection results} holds for $\deltah=\delta/2$, i.e., with probability $1-\deltah$ for all $i\in\brk*{n}$ we have
	$\abs*{ \Ft{x_{g_i}}{c_i} - \Fo{x_0} - \frac{1}{N} \sum_{j=1}^{N} \prn*{ \fl{x_{g_i}}{c_i}{s_j} -  \ft{x_0}{s_j} } } \le \tau_{i}$.
\end{restatable}

We now state a high-probability guarantee on the loss of the output of \Cref{alg: clipping around initial}.
In \Cref{sec: conv rate}, we combine this guarantee with the guarantee for clipped SGD to derive a near-optimal optimality gap.
\begin{proposition}
	\label{thm: clipped model subopt}
	Let $\delta\in\prn*{0,1}$, and let the number of validation samples be $N\ge4 \ln \prn*{ 4/\delta}$.
	Let the candidates $\crl*{x_g}_{g\in\Grid}$ be independent of the validation samples $\crl*{s_i}_{i\in\brk*{N}}$.
	Let $m\in\IsLip[\lell][\infty][\normo{\cdot}]$ and $h\in\ILip[0][1][\normbs{\cdot}]$.
	If $x,c$ are the output of \Cref{alg: clipping around initial}, then with probability at least $1-\delta$ we have
	\begin{align*}
		\Ft{x}{c} \le \min_{g \in \Grid\cup\crl*{0}} \brk*{ \Fo{x_g} + \O{ \Lopt \norma{x_g - x_0} \sqrt{\frac{\lnp\prn*{ \tilde{a} / \delta} \lnp\prn*{ \tilde{a}} }{N} } } },
	\end{align*}
	where $\tilde{a} = \abs{\Grid} \lnp\prn*{\sqrt{N} \cdot \Lopt/\lell}$.
\end{proposition}

\subsection{Warm-up: the Lipschitz case}
The proof of the following proposition is simpler than the proof of \Cref{thm: clipped model subopt}, and can be seen as a warm-up for the latter.
Before proving \Cref{thm: clipped model subopt}, we first consider the simpler setting in which the model $m$ is Lipschitz, i.e., $m\in\ILip[\lell][\infty][\normo{\cdot}]$.
\begin{proposition}
	\label{thm: clipped model subopt lip}
	Let $\delta\in\prn*{0,1}$, and let the number of samples be $N\ge2$.
	Let the candidates $\crl*{x_g}_{g\in\Grid}$ be independent of the validation samples $\crl*{s_i}_{i\in\brk*{N}}$.
	Let $m\in\ILip[\lell][\infty][\normo{\cdot}]$ and $h\in\ILip[0][1][\normbs{\cdot}]$.
	If $x,c$ are the result of \Cref{alg: clipping around initial}, where we use $n=1$, then with probability at least $1-\delta$ we have
	\begin{align*}
		\Ft{x}{c} \le \min_{g \in \Grid\cup\crl*{0}} \brk*{ \Fo{x_g} + \O{ \Lopt \norma{x_g - x_0} \sqrt{\frac{\lnp\prn*{ \tilde{a} / \delta} }{N} } } },
	\end{align*}
	where $\tilde{a} = \abs{\Grid} \lnp\prn*{\Lopt/\lell}$.
\end{proposition}

\begin{proof}
	For all $g\in\Grid$, let $c_{g,0}\defeq \al[L\opt \norma{x_g - x_0}][\lell \norma{x_g - x_0}]$.
	This value is an intermediate quantity through which we relate $\Ft{\cdot}{c_{g,0}}$ to $\Fo{\cdot}$ without affecting the algorithm.
	In this proof, we consider an augmented candidate set $\crl*{\prn*{ x_g, c_{g,i} }}_{\prn*{g,i}\in \Grid \times \prn*{ \brk*{n}\cup\crl*{0} }}$ as the set of candidates passed to \Crefn{alg: reliable model selection}.
	For candidate $\prn*{x_g, c_{g,0}}$, $\tau_{g,0}$ is calculated in the same way as for the other candidates, i.e., $\tau_{g,0}=\te{x_g}{c_{g,0}}$.
	We define $c_{0,0}\defeq0$ and $\tau_{0,0}\defeq0$ for convenience, noting that $\ml{x_0}{c}{s}=\mt{x_0}{s}$ and $\fl{x_0}{c}{s}=\ft{x_0}{s}$ for every $c\ge0$ and $s\in\S$.
	We first show that \Cref{alg: reliable model selection} returns the same output for both the original set of candidates and the modified set of candidates.
	Because $m\prn{\cdot, s}$ is $L\opt$-Lipschitz, \Cref{lem: m diff Lip} states that for any $g\in\Grid$ and sample $s\in \S$ we have that
	\begin{align}
		\label{eq: bound m by c_0}
		\normb{\mt{x_{g}}{s} - \mt{x_0}{s}}
		\le L\opt \norma{x_{g} - x_0}
		\le c_{g,0}.
	\end{align}
	Thus, for any $g\in\Grid$ we have $c_{g,0} \ge c_{g,1}$, and consequently $\tau_{g,0} \ge \tau_{g,1}$.
	Moreover, the validation losses under the clipping values $c_{g,0}$ and $c_{g,1}$ coincide: $\frac{1}{N} \sum_{j=1}^{N} \fl{x_g}{c_{g,0}}{s_j} = \frac{1}{N} \sum_{j=1}^{N} \fl{x_g}{c_{g,1}}{s_j}$.
	\Cref{alg: reliable model selection} first constructs a reliable set $\F$ of candidates that minimize the sum of validation loss and the risk term $\tau$, and then selects from $\F$ the candidate with minimal validation loss.
	Therefore, for every $g\in\Grid$, the candidate $(x_g,c_{g,1})$ dominates $(x_g,c_{g,0})$ in the sense that it has no larger validation loss and a smaller (or equal) risk term.
	It follows that the candidates $\{(x_g,c_{g,0})\}_{g\in\Grid}$ do not affect \Cref{alg: reliable model selection}.
	Hence, the algorithm returns the same output for the original and the modified candidate sets.

	We now bound $\Ft{x_{g}}{c_{g,0}} + \tau_{g,0}$ for any $g\in \Grid$.
	As a direct result of \Cref{eq: bound m by c_0}, we have $\Ft{x_{g}}{c_{g,0}} = \Fo{x_{g}}$.
	In addition, because $\lt{\cdot }{ s}$ is 1-Lipschitz, $\ve{x_g}{c_{g,0}} \le \O{c_{g,0}^2} = \O{\prn*{\Lopt \norma{x_g - x_0}}^2}$.
	Therefore, $\tau_{g,0} \le \O{ \Lopt \norma{x_g - x_0} \sqrt{\frac{\lnp\prn*{\tilde{a} / \delta}}{N}} }$.
	Thus, by combining the previous equation with the fact that $\Ft{x_{g}}{c_{g,0}} = \Fo{x_{g}}$, we obtain
	\begin{align}
		\label{eq: correct LD loss bound}
		\Ft{x_{g}}{c_{g,0}} + 2 \tau_{g,0}
		\le \Fo{x_{g}} + \O{ \Lopt \norma{x_g - x_0} \sqrt{\frac{\lnp\prn*{\tilde{a} / \delta}}{N}} }.
	\end{align}

	Finally, as \Cref{lem: condition hold for algorithm} shows that the requirements of \Cref{lem: reliable selection results} hold for $\delta/2$, \Cref{lem: reliable selection results} guarantees that if $x,c$ is the result of \Cref{alg: clipping around initial}, then with probability at least $1-\delta/2$,
	\begin{align*}
		\Ft{x}{c}
		\le \min_{\substack{i \in \brk*{n} \cup \crl*{0}\\g\in \Grid\cup\crl*{0}}} \brk*{ \Ft{x_{g}}{c_{g,i}} + 2 \tau_{g,i}}.
	\end{align*}
	Hence, as the additional candidates $\crl*{\prn*{x_g, c_{g,0}} }_{g\in\Grid\cup\crl*{0}}$ do not affect the algorithm, \Cref{eq: correct LD loss bound} implies \Cref{thm: clipped model subopt lip}.
\end{proof}

\subsection{Proof outline for \Cref{thm: clipped model subopt}}
The proof for the case $m\in\IsLip[\lell][\infty][\normo{\cdot}]$ follows steps similar to the proof for the case $m\in\ILip[\lell][\infty][\normo{\cdot}]$.
However, some changes are needed to adapt the proof to the case $m\in\IsLip[\lell][\infty][\normo{\cdot}]$.
We now present an outline of the proof.

Unlike the Lipschitz case, in the second-moment-Lipschitz case the quantity $\normb{\ml{x_{g_i}}{c_{i}}{s} - \mt{x_0}{s}}$ may be unbounded with non-negligible probability.
Thus, for any choice of $c_{g,0}$ we cannot guarantee that $c_{g,0}$ bounds $\normb{\ml{x_{g_i}}{c_{i}}{s} - \mt{x_0}{s}}$.
Therefore, unlike what we showed for the Lipschitz case, we cannot conclude that $\Ft{x_{g}}{c_{g,0}} = \Fo{x_{g}}$ in the second-moment-Lipschitz case.
Instead, to balance the trade-off between a small value of $\abs*{\Ft{x_{g}}{c_{g,0}} - \Fo{x_{g}}}$ and a small risk term $\tau_{g,0}$, we choose $c_{g,0}\defeq \al[\sqrt{\frac{N}{\ln \prn*{ 4/\delta }}} L\opt \norma{x_{g} - x_0}][\lell \norma{x_{g} - x_0}]$.
Leveraging the facts that $\lt{\cdot}{s}$ is 1-Lipschitz, that $m$ is $L\opt$-second-moment-Lipschitz, and our choice of $c_{g,0}$, we obtain that $	\abs{ \Ft{x_{g}}{c_{g,0}} - \Fo{x_{g}} } \le  L\opt \norma{ x_{g} - x_{0} } \sqrt{\frac{\ln \prn*{ 4/\delta }}{N}}$, which is sufficiently small.

By using Theorem 10 of \citet{maurer2009empirical} (see \Cref{thm: maurer 10}), with high probability we bound the empirical variance $\sqrt{\ve{x_g}{c_{g,0}}}$ using the true variance and the clipping parameter.
Thus, our choice of $c_{g,0}$ results in a sufficiently small risk term: $\tau_{g,0} \le \O{ \Lopt \norma{ x_{g} - x_{0} } \sqrt{\frac{\ln\prn*{16 \tilde{a} / \delta}}{N}} \sqrt{ \frac{\ln\prn*{16 \tilde{a} / \delta}}{\ln\prn*{4 / \delta}} }}$.

The above suffices to prove that \Crefn{alg: reliable model selection} selects a good candidate from $\crl*{\prn*{ x_g, c_{g,i} }}_{\prn*{g,i}\in \Grid \times \prn*{ \brk*{n}\cup\crl*{0} }}$, similarly to how it was proven in the Lipschitz case.
All that is left is to show that adding the candidate $\prn*{x_{g}, c_{g,0}}$ does not affect \Crefn{alg: reliable model selection}.
We show that since $n=\ceil{4 \ln \prn*{ 4/\delta}}$ is sufficiently large, $c_{g,0}\ge c_{g,n}$.
This results in the validation loss for both clipping parameters $c_{g,0}$ and $c_{g,n}$ being very close:
\begin{align*}
	\abs*{ \frac{1}{N} \sum_{i=1}^{N} \fl{x_{g}}{c_{g,0}}{s_i} - \frac{1}{N} \sum_{i=1}^{N} \fl{x_{g}}{c_{g,n}}{s_i} }
	&\le 2 c_{g,0} \frac{ n-1 }{ N }
	\le 16 \Lopt \norma{x_{g} - x_0} \sqrt{\frac{ \ln \prn*{ 4/\delta} }{ N }}.
\end{align*}
As a consequence, we artificially increase the value of $\tau_{g,0}$ by $16 \Lopt \norma{x_{g} - x_0} \sqrt{\frac{ \ln \prn*{ 4/\delta} }{ N }}$, maintaining a small enough risk term $\tau_{g,0}$.
Thus, we obtain that $\frac{1}{N} \sum_{i=1}^{N} \fl{x_{g}}{c_{g,n}}{s_i} + \tau_{g,c_{g,n}} \le \frac{1}{N} \sum_{i=1}^{N} \fl{x_{g}}{c_{g,0}}{s_i} + \tau_{g,c_{g,0}}$.
This means two things:
(i) The safe set $\F$ calculated in \Crefn{alg: reliable model selection} remains unchanged, except that $\prn*{x_{g}, c_{g,0}}$ may also be included in $\F$, even though we added $\prn*{x_{g}, c_{g,0}}$ as a candidate.
(ii) If the candidate $\prn*{x_{g}, c_{g,0}}$ is in $\F$, then $\prn*{x_{g}, c_{g,n}}$ is also in $\F$.
Now, as we can generalize \Cref{lem: reliable selection results} to hold for any candidate in $\F$ and not only to the result of \Crefn{alg: reliable model selection} (see \Cref{lem: reliable selection results gen}), we can assume that the additional candidate $\prn*{x_{g}, c_{g,0}}$ does not affect \Crefn{alg: reliable model selection}.

Finally, by combining all the above results as in the Lipschitz case, we conclude \Cref{thm: clipped model subopt}.
We present the full proof in \Cref{sec: sm-lip selection theorem}.

\newcommand{\Topt}{T_{\text{opt}}}
\newcommand{\Tsel}{T_{\text{select}}}
\subsection{A parameter-free guarantee for our model selection method}
\label{sec: conv rate}

To obtain a final optimality gap, we combine \Cref{alg: clipping around initial} together with grid search and a base optimization algorithm.
Let $\Grid$ be the set of hyperparameters over which we perform grid search, and $\crl*{x_g}_{g\in\Grid}$ be the solutions obtained by optimizing the objective for each $g\in\Grid$.
We reuse the same $N$ samples across all $\abs{\Grid}$ runs of the optimization algorithm, yielding a total optimization complexity of $\Topt =  \OT{\abs{\Grid} N}$.

\Cref{alg: clipping around initial} uses additional $N$ validation samples.
However, \Cref{alg: clipping around initial} does not calculate gradients, and thus we must quantify the evaluation complexity of \Cref{alg: clipping around initial}.
For \Cref{alg: clipping around initial}, we consider all operations aside from the model evaluation as negligible.
We thus take $\Tsel = \OT{\abs{\Grid} N}$ as the evaluation complexity of \Cref{alg: clipping around initial}; we further justify taking $\Tsel = \OT{\abs{\Grid} N}$ in \Cref{sec: discussion evaluation complexity}.
Finally, we take $T = \Topt + \Tsel = \OT{\abs{\Grid} N}$ as the total evaluation complexity of the combined procedure.

We note that \Cref{alg: clipping around initial} requires a sample size $N \ge 4 \ln\prn*{ 4/\delta}$ to satisfy the requirements of \Cref{thm: clipped model subopt}.
If $N$ is below this threshold, we default to outputting the initial weight $x_0$.
This is justified when $N\le \alpha$ (for some $\alpha>0$), since:
\begin{align}
	\label{eq: initial weight gap}
	\Fo{x_0} - \Fo{\xopt}
	=
	\E_{s\sim \PP} \brk*{ \ft{x_0}{s} - \ft{\xopt}{s} }
	\le L\opt \norma{x_0 - \xopt}
	\le L\opt \norma{x_0 - \xopt} \sqrt{\alpha/N}.
\end{align}

We now informally state the optimality gap obtained by combining \Cref{alg: clipping around initial} together with clipped SGD.
\begin{restatable}{corollary}{rescor}
	\label{cor: informal second-moment suboptimality}
	Let $\norma{\cdot}=\norm{\cdot}_2$ be the Euclidean norm.
	Let $\delta\in(0,1/2)$, $m\in\IsLip[\lell][\uell][\normo[2]{\cdot}]$, $h\in\ILip[0][1][\normbs{\cdot}]$, and $f\in \IDist[\lrho][\urho][\norma[2]{\cdot}]$.
	By grid searching over $\grid{\lell}{\uell} \times \grid{\lrho}{\urho}$, running clipped SGD (see \citet{carmon2024price}) with suitable parameters for each grid point, using the same $N$ samples in each run, and then applying \Cref{alg: clipping around initial}, we obtain $x,c$ such that, with probability at least $1-2\delta$,
	\begin{align*}
		\Ft{x}{c} - \Fo{\xopt}
		\le \O{ \Lopt \Dopt \sqrt{\frac{\lnp\prn*{ \tilde{a} / \delta}\lnp\prn*{ \tilde{a}} }{N} } }
		= \O{ \Lopt \Dopt \sqrt{\frac{\lnp\prn*{ \tilde{a} / \delta} \lnp\prn*{ \tilde{a}} \lnp\prn*{\uell/\lell} \lnp\prn*{\urho/\lrho} }{T} } },
	\end{align*}
	where $\tilde{a} = \lnp\prn*{\uell/\lell} \lnp\prn*{\urho/\lrho} \lnp\prn*{\sqrt{N} \cdot \Lopt/\lell}$.
\end{restatable}
\Cref{cor: informal second-moment suboptimality} follows almost immediately from \Cref{thm: clipped model subopt} and the guarantee for clipped SGD \citep[Proposition 1a]{carmon2024price}.
We prove \Cref{cor: informal second-moment suboptimality} in \Cref{sec: additional upper bound}.
Additionally, in \Cref{sec: additional upper bound}, we derive slightly better optimality gaps for the Lipschitz setting.

\newcommand{\DD}{\mathcal{D}}
\newcommand{\A}{\textbf{A}}
\newcommand{\xp}{x_{\text{perfect}}}

\newcommand{\phig}{\phi}
\newcommand{\phia}{\phi}

\section{A parameter-free method with the best possible sample complexity}
\label{sec: sample algorithm}

This section focuses on producing a parameter-free method that matches the sample complexity of optimization with known parameters.
Sample complexity refers to the number of function samples that the algorithm uses, and can be contrasted with stochastic gradient evaluation complexity (the subject of the previous section), which focuses on the number of stochastic gradient evaluations.
The sample complexity perspective allows algorithms to perform operations that may be computationally infeasible in practice, such as computing an exact minimizer of the empirical objective over the $N$ samples.
In other words, sample-efficient algorithms make more efficient use of small quantities of data, whereas computationally efficient algorithms make more efficient use of compute.

We introduce \Cref{alg: sample complexity clipping around initial} and analyze its optimality gap in \Cref{thm: sample algorithm} under the sample complexity setting.
Our algorithm achieves the best possible optimality gap without requiring \emph{any} prior knowledge of either the Lipschitz constant $L\opt$ or the initial distance to the optimum $D\opt$.

\Cref{alg: sample complexity clipping around initial} first uses $N$ samples from $\PP$ to estimate a lower bound $\fL$ on the Lipschitz constant $L\opt$.
It then draws $2N$ additional samples from $\PP$, conditioned on $\mt{\cdot}{s}$ being $\fL$-Lipschitz.
Following \citet{lawrence2025sample}, the algorithm finds a weight with a small optimality gap under this conditioned distribution.
Finally, it outputs this weight together with a clipping parameter to guarantee a small optimality gap under the original distribution $\PP$.

		\begin{algorithm2e}[htb]
			\caption{Optimization with model output clipping}
			\label{alg: sample complexity clipping around initial}
			\LinesNumbered
			\DontPrintSemicolon
			\Input{
				Initialization $x_0$, number of samples $N$, constants $\delta, \phi$, $\psi_\alpha$, and optimization algorithm \A~mapping a ball radius and $N$ samples to an approximate minimizer of $F$ in the ball.
			}

			\lIf*{ $N\le7$ or $\ln\prn*{1/\delta} > \frac{N}{2e}$ }{\Output{ $x_0,0$. }}
			Compute the maximum Lipschitz constant $\fL$ of $\crl*{\mt{\cdot}{s_i}}_{i=1}^{N}$.\\
			Compute $\lambda = 9\cdot66\phi\frac{\fL\sqrt{\ln\frac{\psi_\alpha}{\delta}}}{\sqrt{N}}$.\\
			$k \leftarrow 0$.\;
			\For{$j = 1, \dots, 5N $}{
				Draw $s \sim \PP$,
				\lIf*{$\mt{\cdot}{s}$ is $\fL$-Lipschitz}{
					$k \leftarrow k + 1$, $s_{N+k} \leftarrow s$.\;
				}
			}
			\lIf*{$k < 2N$}{\Output{\textsc{failure}.}}
			Choose $\hat{x}_{\lambda} \in \argmin_{x\in\X}\frac{1}{N}\sum_{i=N+1}^{2N} \ft{x}{s_i} + \lambda\norma{x-x_0}$.\\
			Set $\xp \leftarrow \A \prn*{\fL, 3\norma{\hat{x}_{\lambda} - x_0} ; s_{2N+1}, \dots, s_{3N} }$.

			\Output{ $\xp, \fL \norma{\xp - x_0}$. }
		\end{algorithm2e}

For the algorithm to work properly, we make two assumptions.
First, given a radius $\hat{R}$ and the Lipschitz constant $\fL$, the optimization algorithm \A~must be able to achieve a small optimality gap.
\begin{assumption}
	\label{asm: alg subopt}
	There exist constants $\phia\ge1$ and $\psi_\alpha\ge1$ such that for all $\fL>0$, $f\in\IDist[0][\infty]$, $\hat{R}\ge0$, $\delta\in\prn*{0,\frac{1}{5}}$, and any distribution $\tilde{\PP}$ over $\S$ such that $f\prn*{\cdot;s}$ is $\fL$-Lipschitz a.s., $\A$ satisfies
	\begin{align*}
		&\mathop{\P}_{s_1,\dots,s_N\iid\tilde{\PP}}\brk*{ \fFo{\A\prn*{\fL, \hat{R};s_{1},\dots,s_{N}}} - \min_{x\in\X : \norma{x - x_0}\le\hat{R}} \fFo{x} \le \phia\frac{\fL\hat{R}\sqrt{\ln\frac{\psi_\alpha}{\delta}}}{\sqrt{N}} } \ge 1 - \delta, \,\, \text{and}\\
		&\mathop{\P}_{s_1,\dots,s_N\iid\tilde{\PP}}\brk*{\A\prn*{\fL, \hat{R};s_{1},\dots,s_{N}}\in\crl*{x\in\X ~|~ \norma{x-x_0}\le\hat{R}}} = 1,
	\end{align*}
	where $\fFo{x} \defeq \int_{s\sim\tilde{\PP}} f\prn*{x;s}$.
\end{assumption}
Second, similarly to \citet{lawrence2025sample}, we assume that the empirical mean of the stochastic gradients $\nabla\ft{x}{s}$ is a good approximation to the true gradient $\nabla \Fo{x}$.
\begin{assumption}
	\label{asm: grad approx}
	There exist $\phig\ge1$ and $\psi_\alpha\ge1$ such that for any $f\in\Idiff$, any $\delta\in\prn*{0,\frac{1}{5}}$, any $x\in\X$, any $\fL\ge0$, and any distribution $\tilde{\PP}$ over $\S$ such that $f\prn*{\cdot;s}$ is $\fL$-Lipschitz a.s., we have
	\begin{align*}
		\mathop{\P}_{s_1,\dots,s_N\iid\tilde{\PP}}\brk*{ \normas{\nabla \E_{s\sim\tilde{\PP}} \ft{x}{s} - \frac{1}{N} \sum_{i=1}^{N} \nabla \ft{x}{s_i}} >  594 \cdot \frac{\phig}{2} \cdot \frac{\fL\sqrt{\ln\frac{\psi_\alpha}{\delta}}}{\sqrt{N}} } \le \delta.
	\end{align*}
\end{assumption}

\Cref{asm: alg subopt} is known to hold for many optimization algorithms with $\phi=\O{1}$.
Specifically, if $\norma{\cdot}=\norm{\cdot}_1$, then \Cref{asm: alg subopt} holds for $\psi_\alpha=d$ with entropic mirror descent (i.e., mirror descent with KL divergence) \citep{beck2003mirror,nemirovski1983problem} with adaptive step sizes~\citep{orabona2021modern}.
Additionally, for the Euclidean norm $\norma{\cdot}=\norm{\cdot}_2$, the assumption holds $\psi_\alpha=1$ with \textsc{Adaptive SGD} \citeg{gupta2017unified}.

Similarly, \Cref{asm: grad approx} holds with $\phi=\O{1}$ for $\norma{\cdot}=\norm{\cdot}_1$ and $\norma{\cdot}=\norm{\cdot}_2$, with $\psi_\alpha=d$, and $\psi_\alpha=1$, respectively.
We provide the proof that \Cref{asm: grad approx} holds in these cases in \Cref{sec: stochastic gradients approx gradient}.

When both assumptions hold, we get the following optimality gap guarantee.
For the Euclidean norm $\norma{\cdot}=\norm{\cdot}_2$, and when the optimization algorithm \A~is \textsc{Adaptive SGD}, this matches the best possible sample complexity guarantee \citep[Proposition 1b]{carmon2024price}.

\begin{restatable}{theorem}{thmsamplealg}
	\label{thm: sample algorithm}
	Let $\delta\in\prn*{0,\frac{1}{5}}$, $m\in\ILip[0][\infty][\normo{\cdot}]$, $h\in\ILip[0][1][\normbs{\cdot}]$ and $f\in\IDistf[0][\infty]$.
	If \Cref{asm: alg subopt,asm: grad approx} hold with $\phi$ and $\psi_\alpha$, and $\xp, \fL \norma{\xp - x_0}$ are the outputs of \Cref{alg: sample complexity clipping around initial}, then with probability at least $1-5\delta$ we have
	\begin{align*}
		\Ft{\xp}{\fL\norma{\xp-x_0}}
		&\le \Fo{\xopt} + \O{ \phi\frac{L\opt\norma{\xopt - x_0}\sqrt{\ln\frac{\psi_\alpha}{\delta}}}{\sqrt{N}}}.
	\end{align*}
\end{restatable}

\paragraph{Proof outline for \Cref{thm: sample algorithm} (full proof in \Cref{sec: sample algorithm proof}).}
We define $L\prn*{s}$ as the Lipschitz constant of $\mt{\cdot}{s}$, and $L\prn*{\PP}$ as the distribution over the Lipschitz constants resulting from the distribution $\PP$.
We prove that with high probability $\fL = \max_{i\in\brk*{N}} L\prn*{s_i}$ is greater than the quantile $\Q{L\prn*{\PP}}{1-\ln\prn{1/\delta}/N}$.
This means that for most samples $s\sim\PP$, $\mt{\cdot}{s}$ is $\fL$-Lipschitz.
As a consequence, due to the model clipping, the samples $s\in\S$ for which $\mt{\cdot}{s}$ is not $\fL$-Lipschitz increase the optimality gap by at most $\prn*{\fL \norma{\xp - x_0} + L\opt \norma{\xopt - x_0}}\frac{\ln(1/\delta)}{N}$.

We further show that among $5N$ samples drawn from $\PP$, at least $2N$ have Lipschitz constant below $\Q{L\prn*{\PP}}{1-\ln\prn{1/\delta}/N}$ with high probability, implying that \Cref{alg: sample complexity clipping around initial} does not output failure.

Define $\fF\prn{x} \defeq{} \E_{\substack{s\sim\PP}} \brk*{ \ft{x}{s} ~|~ L\prn*{s} \le \fL}$.
We denote $\hat{R} \defeq 3 \norma{\hat{x}_{\lambda} - x_0}$.
Additionally, for any $\theta\ge0$, we define $\hat{x}_{\theta}^{\star} \in \argmin_{x \in \X} \crl*{ \fFo{x} + \theta \norma{x-x_0} }$.
Using Lemma 2 of \citet{lawrence2025sample}, we show that there exists $\lambda_{\hat{R}} \in \brk*{ \lambda/3, 3\lambda }$ such that $\norma{\hat{x}_{\lambda_{\hat{R}}}^{\star} - x_0} \in \brk*{ \frac{\hat{R}}{33}, \hat{R} }$.
Now, similarly to \citet{lawrence2025sample}, we obtain
\begin{align}
	\label{eq: R simplified main}
	\min_{x\in\X : \norma{x - x_0}\le\hat{R}} \fFo{x} \le \fFo{\hat{x}_{\lambda_{\hat{R}}}^{\star}} \le  \fFo{\xopt} + \lambda_{\hat{R}} \prn*{ \norma{\xopt - x_0} - \norma{\hat{x}_{\lambda_{\hat{R}}}^{\star} - x_0} }.
\end{align}

In the case that $\fL \hat{R} \le \O{L\opt \norma{\xopt - x_0}}$, i.e. $\fL \hat{R}$ is small enough relative to $\O{L\opt \norma{\xopt - x_0}}$, we continue in a similar manner to \citet{lawrence2025sample} and use \Cref{asm: alg subopt} to prove the final optimality gap.
However, $\fL \hat{R}$ can be significantly larger than $\O{L\opt \norma{\xopt - x_0}}$.
In this case, we exploit the negative term $-\lambda_{\hat{R}} \norma{\hat{x}_{\lambda_{\hat{R}}}^{\star} - x_0}$ on the right-hand side of \Cref{eq: R simplified main} to cancel the positive terms in the bound, and thus, with high probability $\Ft{\xp}{\fL\norma{\xp-x_0}}\le \Fo{\xopt}$.
As a result, we obtain the required optimality gap for all cases.

\section{Lower bounds for the model--loss setup}
\label{sec: lower bounds}

Our algorithms succeed in achieving a near-optimal optimality gap of $\Otil{L\opt D\opt\sqrt{\ln\fracs{1}{\delta}}/\sqrt{N}}$,
surpassing the lower bounds in the case of uncertainty in both the Lipschitz constant and the distance from the initial point to the optimum.
We achieve this because we assume that the optimized function can be decomposed into a model and a loss, and introduce a method to clip the model output.
Therefore, at first glance, it is not clear that any of the known lower bounds apply.

However, in some cases, we can decompose a function $f$ into a model $m$ and a loss function $h$ such that all the complexity of $f$ is inherited by $h$.
In particular, consider lower-bound constructions in which each function $f_i$ is $L\opt$-Lipschitz for a known constant $L\opt$.
For each such $f_i$, we decompose $f_i$ into $m_i\prn*{x;s}\defeq x L\opt$ and $h_i\prn*{x;s}\defeq f_i\prn*{x / L\opt;s}$.
Under this decomposition, each $m_i$ is $L\opt$-Lipschitz, and each $h_i$ is $1$-Lipschitz.
Moreover, having the ability to modify the output of $m_i\prn*{x;s}$ is equivalent to modifying the input $x$.
Consequently, any lower-bound argument that applies to the original functions $f_i$ continues to apply in the model--loss setting.
Therefore, the lower bound $\Om{L\opt D\opt\sqrt{\ln\fracs{1}{\delta}}/\sqrt{N}}$ established in \citet[Proposition 1b]{carmon2024price} for the case where $L\opt$ and $D\opt$ are known also holds in our setting.
Hence, the optimality gaps established in this work are still near-optimal even in the model--loss setting.

In contrast, not all lower bounds carry over to the model--loss setting.
For instance, the lower bounds established in \citet[Theorem 3]{carmon2024price} no longer apply, as we obtain strictly better optimality gaps in this work.
In \citet[Theorem 3]{carmon2024price}, the Lipschitz constant $L\opt$ is unknown, and only coarse lower and upper bounds on $L\opt$ are assumed.
Their construction involves two types of functions: one with a uniformly small Lipschitz constant, and another that has a small Lipschitz constant for most samples but a large Lipschitz constant on a rare subset of samples.
Consequently, the reduction argument from the previous paragraph does not apply to this setting.
In our framework, however, the ability to clip the model output limits the loss contribution of rare high-Lipschitz samples.
This mechanism prevents the adversarial behavior exploited in \citet[Theorem 3]{carmon2024price}, enabling us to achieve near-optimal optimality gaps despite the unknown Lipschitz constant.

In \Cref{sec: Formal lower bounds}, we formally state the lower bound $\Om{L\opt D\opt\sqrt{\ln\fracs{1}{\delta}}/\sqrt{N}}$ that carries over to the model--loss setting.

\arxiv{
\newcommand{\acks}[1]{
	\section*{Acknowledgments}
	#1
}
}

\acks{
	This work was supported by the NSF-BSF program, under NSF grant \#2239527 and BSF grant \#2022663.
	YC acknowledges support from the Israeli Science Foundation (ISF) grant no. 2486/21.
}

\newpage
\arxiv{\bibliographystyle{abbrvnat}}

\appendix

\notarxiv{
\crefalias{section}{appendix} %
}

\newpage
\tableofcontents
\newpage

\section{Reliable model selection method}
\label{sec: reliable model selection}
		\begin{algorithm2e}[H]
			\caption{Reliable model selection method}
			\label{alg: reliable model selection}
			\LinesNumbered
			\DontPrintSemicolon
			\Input{
				A function $f$,
				a scalar $\gamma\in\brkprn{1,\infty}$,
				candidate solutions $x_0,x_1,\dots,x_K$,
				samples $s_1, s_2,\dots,s_N$,
				and scalars $\tau_{1}, \dots, \tau_{K}$.
			}
			$\tau_{0} \defeq 0$\\
			$\theta = \min_{k\in\brk*{K}\cup\crl*{0}} \prn*{ \gamma \tau_{k} + \frac{1}{N}\sum_{i=1}^{N} \ft{x_k}{s_i} }$.\\
			$\F = \crl*{k\in\brk*{K}\cup\crl*{0} ~|~ \tau_{k} + \frac{1}{N}\sum_{i=1}^{N} \ft{x_k}{s_i} \le \theta}$.\\
			$k_\text{reliable} = \argmin_{k\in\mathcal{F}} \frac{1}{N}\sum_{i=1}^{N} \ft{x_k}{s_i}$.\\
			\Return $x_{k_\text{reliable}}$.
		\end{algorithm2e}

Throughout this paper, for simplicity, we use $\gamma=1$.
However, the result can easily be expanded to any fixed $\gamma\in\brkprn{1,\infty}$.

\subsection{Generalization of \Creflem{lem: reliable selection results}}

\Cref{lem: reliable selection results} provides a high-probability guarantee on how much worse the result of \Cref{alg: reliable model selection} can be relative to the other candidates.
We generalize the guarantee to all candidates in $\F$.
\begin{lemma}
	\label{lem: reliable selection results gen}
	Let $\deltah\in\prn*{0,1}$.
	Let $K$ be the number of models, and $N$ be the validation sample set size.
	Let $\tau_{1}, \dots, \tau_{K}$ be non-negative scalars such that with probability $1-\deltah$ for all $k\in\brk{K}$ we have
	$\abs*{ \Fo{x_k} - \Fo{x_0} - \frac{1}{N} \sum_{i=1}^{N} \prn*{ \ft{x_k}{s_i} -  \ft{x_0}{s_i} } } \le \tau_{k}$.
	Let $\gamma\in\brkprn{1,\infty}$, and let $\F$ be the set $\F$ calculated by \Cref{alg: reliable model selection} with $\gamma$.
	Then, with probability $1-\deltah$, for any $k_\F\in\F$ we have
	\begin{align*}
		\Fo{x_{k_\F}} \le \min_{k\in\brk*{K}\cup\crl*{0}} \prn*{ \Fo{x_k} + \prn*{1 + \gamma}\tau_{k} }.
	\end{align*}
\end{lemma}

The proof of \Cref{lem: reliable selection results gen} is the same as the proof of \Cref{lem: reliable selection results} presented in \cite{lawrence2025sample}.
For completeness, we reproduce that proof here.
\begin{proof}
	Let $k_\F$ be any element of $\F$.
	Define
	\begin{align*}
		\barF{k} = \frac{1}{N}\sum_{i=1}^{N} \ft{x_k}{s_i}.
	\end{align*}
	For all $k \in \brk*{K}$ we have
	\begin{align*}
		0
		&\overle{(i)} \theta -  \barF{k_\F}  - \tau_{k_\F}
		\overle{(ii)} \barF{k} + \gamma \tau_k - \barF{k_\F}  - \tau_{k_\F} \\
		&=  \gamma \tau_k + \barF{k}  - \barF{0} + \barF{0} - \barF{k_\F}  - \tau_{k_\F}
		\overle{(iii)}  (1 + \gamma) \tau_k + \Fo{x_k} - \Fo{x_{k_\F}},
	\end{align*}
	where $(i)$ follows from $k_\F \in \mathcal{F}$, $(ii)$ follows from the definition of $\theta$, and $(iii)$ follows from the fact that $\abs*{ \Fo{x_j} - \Fo{x_0} - \prn*{ \barF{j} - \barF{0} } } \le \tau_{j}$ for all $j\in\brk*{K}$.
\end{proof}

\section{Discussion about the evaluation complexity}
\label{sec: discussion evaluation complexity}

Below, we quantify the evaluation complexity of \Cref{alg: clipping around initial}.
\begin{itemize}
	\item In the case that $m\in\ILip[\lell][\infty][\normo{\cdot}]$ we only need to include a single value in $\C_g$.
	In this case, the total evaluation complexity in \Cref{alg: clipping around initial} is $\Tsel = \OT{\abs{\Grid} N}$.

	\item If the grid search is also done on the Lipschitz constant, and one of the generated $x_g$ is associated with $L_g$ that exceeds $L\opt$ by at most a constant factor, then we can modify $\C_g$ to contain only a single clipping parameter without changing the proof of \Cref{thm: clipped model subopt}. We can set $\C_g=\crl*{L_g \norma{x_g - x_0}}$ in the context that $m\in\ILip[\lell][\infty][\normo{\cdot}]$, and $\C_g=\crl*{\sqrt{\frac{N}{\ln \prn*{ 4/\delta }}} L_g \norma{x_g - x_0}}$ in the context that $m\in\IsLip[\lell][\infty][\normo{\cdot}]$.
	With this modification, the evaluation complexity of \Cref{alg: clipping around initial} is $\Tsel = \OT{\abs{\Grid} N}$.

	\item Lastly, not all operations take the same time. Computing gradients and evaluating the output of the model $m$ are considered time-consuming operations. In contrast, obtaining the $\abs{\C_g}$ values needed to form $\C_g$ and evaluating the loss function $h$ take relatively insignificant time.
	Therefore, if we consider only the operations of evaluating the output of the model $m$, the total evaluation complexity in \Cref{alg: clipping around initial} is $\Tsel = \OT{\abs{\Grid} N}$.
	If we do not ignore the remaining operations, then obtaining the values for $\C_g$ takes $\OT{\abs{\Grid} N \log\abs{\C_g}}$.
	Similarly, evaluating all $\tau_{g,i}$ takes $\OT{\abs{\Grid} \prn*{N + \abs{\C_g}^2}}$ since we do not need to recompute everything from scratch for each clipping parameter in $\C_g$.
	Moreover, \Cref{thm: clipped model subopt} continues to hold even if, for each $g\in\Grid$, $\C_g$ contains only the $n$-th largest value of the \emph{multiset} $\crl*{\al[\normb{\mt{x_g}{s_i} - \mt{x_0}{s_i}}][\lell \norma{x_g - x_0}]}_{i=1}^{N}$;
	in this case, there is only one confidence interval $\tau_g$ per grid point, reducing the total time required to compute the confidence intervals to $\OT{\abs{\Grid} N}$.
\end{itemize}
For the reasons above, we take $\Tsel = \OT{\abs{\Grid} N}$ as the evaluation complexity of \Cref{alg: clipping around initial}.

\section{Additional upper bounds}
\label{sec: additional upper bound}

\rescor*

\begin{proof}
	If $N < 4 \ln \prn*{ 4/\delta}$, we return $x_0,0$ instead of performing a grid search.
	In this case, applying \Cref{eq: initial weight gap} gives
	\begin{align*}
		\Ft{x_0}{0} - \Fo{\xopt}
		= \Fo{x_0} - \Fo{\xopt}
		\le
		\O{ \Lopt \Dopt \sqrt{\frac{\lnp\prn*{ 1 / \delta} }{N} } },
	\end{align*}
	which is the desired guarantee.

	Now suppose that $N \ge 4 \ln \prn*{ 4/\delta}$.
	Let the grid be $\Grid=\grid{\lell}{\uell} \times \grid{\lrho}{\urho}$.
	For every $\prn*{L,D} \in \Grid$, let $x_{L,D}$ be the weight obtained by running clipped SGD for $N$ steps with step size
	\begin{align*}
		\eta=\frac{D}{L\sqrt{N}},
	\end{align*}
	gradient clipping threshold
	\begin{align*}
		G_{\mathrm{clip}}
		= \frac{L\sqrt{N}}{\sqrt{\log\prn*{ 1/\delta }}},
	\end{align*}
	Euclidean projection onto
	\begin{align*}
		\X \cap \crl*{x ~|~ \norm{x - x_0}_2\le D},
	\end{align*}
	and the same $N$ samples across all grid points.

	By the definition of $\Grid$, there exists a grid point $\prn*{L,D} \in \Grid$ such that
	\begin{align*}
		\Lopt \le L \le 2\Lopt
		\qquad\text{and}\qquad
		\Dopt \le D \le 2\Dopt .
	\end{align*}
	For this grid point, \citet[Proposition 1a]{carmon2024price} implies\footnote{While \citet{carmon2024price} assume throughout that $\ft{\cdot}{s}$ is convex for all $s\in\S$, the proof of Proposition 1a only uses that $\Fo{\cdot}$ is convex.} that, with probability at least $1-\delta$,
	\begin{align*}
		\Fo{x_{L,D}} - \Fo{\xopt}
		\le
		\O{ \Lopt \Dopt \sqrt{\frac{\lnp\prn*{ 1 / \delta} }{N} } }.
	\end{align*}

	Let $x,c$ be the result obtained by applying \Cref{alg: clipping around initial} to the candidates $\crl*{x_{L,D}}_{\prn*{L,D}\in\Grid}$ using $N$ fresh validation samples.
	Applying \Cref{thm: clipped model subopt} and taking a union bound over the clipped-SGD guarantee and the validation guarantee yields the desired bound on $\Ft{x}{c} - \Fo{\xopt}$.
\end{proof}

We can also use \Cref{alg: clipping around initial} with \citet{carmon2022making} to obtain a slightly different optimality gap in the Lipschitz case.
Compared to the previous bound, this gap depends less on the upper and lower bounds on the Lipschitz constant and distance to the optima $\lell,\uell,\lrho,\urho$, but depends more on $\delta$.
\begin{corollary}
	Let $\norma{\cdot}=\norm{\cdot}_2$ be the Euclidean norm.
	Let $m\in \ILip[\lell][\uell][\normo[2]{\cdot}]$, $h\in\ILip[0][1][\normbs{\cdot}]$ and $f\in \IDist[\lrho][\urho][\norma[2]{\cdot}]$.
	Let $\delta\in\prn*{0,1/2}$.
	Let the grid be $\Grid=\grid[\sqrt{N}]{\lell}{\uell}$.
	For every $L \in \Grid$, let $x_L$ be the resulting weight of running \citet{carmon2022making} with $\eta_{\varepsilon}=\lrho/L$, and the same $N$ samples across all runs.
	If $N\ge2$, let $x,c$ be the result obtained by applying \Cref{alg: clipping around initial}, with $n=1$ and using additional $N$ validation samples, on the candidates $\crl*{x_g}_{g\in\Grid}$; otherwise, let $x,c$ be $x_0,0$. %
	With probability at least $1-2\delta$, we have
	\begin{align*}
		\Ft{x}{c} - \Fo{\xopt}
		&\le \O{ \Lopt \Dopt \frac{\lnp\prn*{\tilde{b}/\delta}\sqrt{\lnp\prn*{ \tilde{a} / \delta}} + \lnp^2\prn*{\tilde{b}/\delta} }{\sqrt{N}} }\\
		&= \O{ \Lopt \Dopt \frac{ \prn*{\lnp\prn*{\tilde{b}/\delta}\sqrt{\lnp\prn*{ \tilde{a} / \delta}} + \lnp^2\prn*{\tilde{b}/\delta}} \sqrt{ \log_{\sqrt{N}}\prn*{e + \uell/\lell} } }{\sqrt{T}} },
	\end{align*}
	where $\tilde{a} = \log_{\sqrt{N}}\prn*{e + \uell/\lell} \lnp\prn*{\Lopt/\lell}$ and $\tilde{b} = \lnp\prn*{N \Dopt/\lrho}$.
\end{corollary}

\begin{proof}
	If $N <2$, then simply applying \Cref{eq: initial weight gap} gives the desired guarantee
	\begin{align*}
		\Ft{x_0}{0} - \Fo{\xopt} = \Fo{x_0} - \Fo{\xopt} \le \O{ \Lopt \Dopt \sqrt{\frac{\lnp\prn*{ 1 / \delta} }{N} } }.
	\end{align*}

	From the definition of $\Grid$, there exists a point in the grid $L \in \Grid$ such that $\Lopt \le L \le \sqrt{N}\Lopt$.
	For this point, \citet[Theorem 2]{carmon2024price} proved that, with probability at least $1-\delta$,
	\begin{align*}
		\Fo{x_{L}} - \Fo{\xopt} \le  \O{ \Lopt \Dopt \frac{ \lnp^2\prn*{\tilde{b}/\delta} }{\sqrt{N}} }.
	\end{align*}
	It additionally proves that $\norm{x_{L} - x_0} \le \O{\Dopt \lnp\prn*{\tilde{b}/\delta}}$.
	Now, by simply applying \Cref{thm: clipped model subopt} and using the union bound, we obtain the desired guarantee on $\Ft{x}{c} - \Fo{\xopt}$.
\end{proof}

\newcommandx{\Alg}[2][1=T, 2=\delta]{\mathcal{A}_{#1}^{#2}}
\newcommandx{\mfunc}[2][1=N, 2=m]{q_{#1}^{#2} }
\newcommandx{\mfunck}[3][1=T, 2=k, 3=m]{q_{#1, #2}^{#3} }
\section{Formally stated lower bound}
\label{sec: Formal lower bounds}

Although our method clips the model output to achieve improved optimality gaps, we establish lower bounds that apply more broadly to any method that modifies the model output.
Specifically, we consider algorithms that, instead of outputting a weight $x\in\X$, output a function $\mfunc[] : \S \rightarrow \Y$ that represents a generic modification to the model output.
In the finite case, we restrict the function to be of the form $\mfunc[]\prn{s}\defeq U\prn{\mt{x_1}{s}, \dots, \mt{x_k}{s}}$ for some $k\in\mathbb{N}$, where $x_1, \dots, x_k\in\X$ and $U:\Y^k \to \Y$ are output by the optimization algorithm.
More generally, we only require that the output function $\mfunc[] : \S \rightarrow \Y$ is in the set
\begin{align}
	\label{eq: q restriction}
	\QQ \defeq \crl*{q\in\S \rightarrow \Y ~|~ \forall s_1,s_2 \in\S: \text{if }\forall x\in\X: \mt{x}{s_1}=\mt{x}{s_2} \text{ then } q\prn*{s_1}=q\prn*{s_2}}.
\end{align}

\subsection{Formal setup}
We now formally define the basic building blocks of this section: stochastic optimization problems, algorithms, error metrics, and minimax rates.
We define them similarly to \citet{carmon2024price}.

\paragraph{Stochastic optimization under model--loss decomposition (SOML) problems.}
A SOML problem instance is a tuple $\prn*{m,h,f,\PP}$ containing a distribution $\PP$ over $\S$, a sample objective $f:\X \times \S \rightarrow \R$, and functions $m:\X \times \S \rightarrow \Y$ and $h:\Y \times \S \rightarrow \R$ such that $\ft{x}{s}\defeq \lt{ \mt{x}{s} }{ s }$.
Let $\ISOML$ denote the class of all SOML problem instances.
We consider two fundamental classes of convex functions with a minimizer at most $D$ away from the initial weight $x_0$.
The first class contains problems for which $m$ is $L$-Lipschitz and $h$ is $1$-Lipschitz.
\begin{align*}
	\ISOMLl \defeq \crl*{ \prn*{m,h,f,\PP} \in \ISOML ~|~ m\in\ILip[0][L][\normo{\cdot}] \text{ and } h\in\ILip[0][1][\normbs{\cdot}] \text{, and } f\in\IDistf[0][D]}.
\end{align*}
The second class contains problems for which $m$ is $L$-second-moment-Lipschitz and $h$ is $1$-Lipschitz,
\begin{align*}
	\ISOMLs \defeq \crl*{ \prn*{m,h,f,\PP} \in \ISOML ~|~ m\in\IsLip[0][L][\normo{\cdot}] \text{ and } h\in\ILip[0][1][\normbs{\cdot}] \text{, and } f\in\IDistf[0][D] }.
\end{align*}

\paragraph{Optimization algorithms.}
We define a SOML algorithm as an algorithm that has unrestricted access to the functions $m,h$ and $f$, but observes $\PP$ only through samples $s_1,\cdots,s_N\iid\PP$, and outputs $\mfunc : \S \rightarrow \Y$ such that $\mfunc \in \QQ$.
Such algorithms are allowed to perform computationally infeasible operations, such as finding the Lipschitz constant of $m\prn{\cdot,s}$ for some $s\in\S$.
Thus, the unrestricted access to the functions $m,h$ and $f$ allows the algorithm to access the $i$th-order derivative for any $i$.
We write $\AA_{\text{SOML}}$ for the set of all SOML algorithms.

\paragraph{Error metrics.} For an algorithm $\alg$, SOML instance $\prn*{m,h,f,\PP}$, and budget $N$, we let $\mfunc$ denote the
algorithm's output given $N$ samples.
We define the high-probability error at confidence level $\delta$ as
\begin{align*}
	\Err{\alg}{\prn*{m,h,f,\PP}} = \Q{\E_{s\sim \PP} \lt{ \mfunc\prn*{s} }{ s } }{1-\delta} - \inf_{\xopt\in\X} \E_{s\sim \PP} \ft{\xopt}{s},
\end{align*}
where $\Q{Y}{p}\defeq\min\crl*{ y : \P\brk*{ Y \le y } \ge p }$ denotes the $p$-th quantile of the random variable $Y$, similar to the definition of $\Q{\cdot}{p}$ over distributions.

\paragraph{Minimax error.}
Given an instance class $\I$ and an algorithm $\alg$, we overload the notation above to denote worst-case error over the class,
\begin{align*}
	\Err{\alg}{\I} \defeq \sup_{\prn*{m,h,f,\PP}\in\I} \Err{\alg}{\prn*{m,h,f,\PP}}.
\end{align*}
Given an algorithm class $\AA$, we further overload our notation to express minimax optimal worst-case error
\begin{align*}
	\Err{\AA}{\I} \defeq \inf_{\alg\in\AA} \Err{\alg}{\I} = \inf_{\alg\in\AA} \sup_{\prn*{m,h,f,\PP}\in\I} \Err{\alg}{\prn*{m,h,f,\PP}}.
\end{align*}

\subsection{Lower bound statement}
We now formally establish the lower bound.
The proof of the bound is a reduction to the lower bounds established in \citet[Proposition 1b and Theorem 2]{carmon2024price}.
\begin{proposition}
	For any $L,D>0$, $\delta\in\prn*{0, \frac{1}{2}}$, and any number of samples $N\ge1$, we have:
	\begin{align*}
		\Err{\AA_{\text{SOML}}}{\ISOMLs}
		\ge \Err{\AA_{\text{SOML}}}{\ISOMLl}
		\ge \Om{L D \min\crl*{\sqrt{\lnp\fracs{1}{\delta}}/\sqrt{N}, 1}}.
	\end{align*}
\end{proposition}

\begin{proof}
	The proof mostly follows from the proof of \citet[Proposition 1b]{carmon2024price}.
	In the proof of \citet[Proposition 1b]{carmon2024price}, they let $\X=\R, \S=\crl*{0,1}$ and the optimized function $f: \X\times\S\rightarrow\R$ be
	\begin{align*}
		\ft{x}{s} = \begin{cases}
			L \abs{x} & $s=0$ \\
			L \abs{x-D} & $s=1$.
		\end{cases}
	\end{align*}
	They consider the distributions $P_v := \text{Bernoulli} \prn*{\frac{1+(2v-1)\varepsilon}{2}}$ for $v\in\crl*{0,1}$ and some $\varepsilon\in\brk{0,\frac{1}{2}}$.

	We set $\Y=\R$ and decompose $f$ into a model $m$ and a loss function $h$ as follows:
	\begin{align*}
		\mt{x}{s} = L x,
	\end{align*}
	and
	\begin{align*}
		\lt{x}{s} = \begin{cases}
			\abs{x} & $s=0$ \\
			\abs{x-L D} & $s=1$.
		\end{cases}
	\end{align*}
	For this decomposition it is easy to see that $\ft{x}{s} = \lt{\mt{x}{s}}{s}$, that $m$ is $L$-Lipschitz, $h$ is $1$-Lipschitz, and that $f\in\IDistf[0][D][\abs{\cdot}]$.

	Let $\mfunc$ be the output of the optimization algorithm after interacting with $m,h$ and $f$ through $N$ samples.
	As $\mt{x}{s}$ depends only on $x$ and does not depend on $s$, $\mfunc\prn*{ s }$ does not depend on $s$.
	As such, we can choose $x_N = \mfunc\prn*{ \cdot } / L$.
	Therefore, $\ft{x_N}{s} = \lt{\mfunc\prn*{ s } }{ s}$.
	Thus, from here we can use the same proof as in the proof of \citet[Proposition 1b]{carmon2024price}.
	Equivalently, any optimization algorithm used for optimizing the problems described here can be used to optimize the problems described in the proof of \citet[Proposition 1b]{carmon2024price}, and vice versa.

	We note that $\abs{\cdot}$ is not differentiable at $0$.
	Using a smoothing argument, the lower bound can also be extended to the differentiable setting.
	For example, we can replace each occurrence of $\abs{\cdot}$ in the construction of $\lt{\cdot}{\cdot}$ with
	\begin{align*}
		\abs{y}_\gamma \defeq \begin{cases}
			\abs{y} & \text{if } y\ge\gamma \text{ or } y\le-\gamma\\
			\frac{y^2}{2\gamma} + \frac{\gamma}{2} & \text{if } y\in\brk*{-\gamma,\gamma}
		\end{cases},
	\end{align*}
	for some $\gamma>0$.
	The modified function is differentiable and differs from the original function by at most $\frac{\gamma}{2}$.
	By choosing $\gamma$ much smaller than $\frac{L D}{2} \min\crl*{\sqrt{\lnp\fracs{1}{\delta}}/\sqrt{N}, 1}$, we obtain that optimizing the modified function is almost equivalent to optimizing the original function.
	Thus, the lower bound also holds for differentiable functions.
\end{proof}

\section{Well-known results}
This section collects several well-known probabilistic and optimization results used throughout the paper.

\subsection{Chebyshev's inequalities}

\begin{theorem}[Chebyshev's inequality]
	\label{thm: Chebyshev}
	Let $X$ be a random variable with finite variance $\sigma^2$.
	For any $k>0$, we have
	\begin{align*}
		\P\brk*{ \abs{X - \E\brk*{X}} > k\sigma } \le \frac{1}{k^2}.
	\end{align*}
\end{theorem}

\begin{theorem}[Exponential Chebyshev's inequality]
	\label{thm: exp Chebyshev}
	Let $X$ be a random variable.
	For any $t,\varepsilon>0$ we have
	\begin{align*}
		\P\brk*{ X \ge \varepsilon } \le e^{-t\varepsilon} \E\brk*{ e^{tX} }.
	\end{align*}
\end{theorem}

\subsection{Concentration bounds}

\newcommand{\bX}{\boldsymbol{X}}

Define the sample variance as
\begin{align*}
	V_n\prn*{\bX}
	\defeq \frac{1}{2 n\prn*{n-1}} \sum_{i=1}^{n}\sum_{j=1}^{n} \prn*{ X_i - X_j }^2
	= \frac{1}{n-1} \sum_{i=1}^{n}\prn*{ X_i - \frac{1}{n}\sum_{j=1}^{n} X_j }^2.
\end{align*}

\begin{theorem}[Theorem 4 of \citet{maurer2009empirical}]
	\label{thm: maurer 4}
	Let $c>0$, $n\ge2$, and $X, X_1,\dots,X_n$ be i.i.d. random variables with values in $\brk*{-c,c}$.
	Define $\bX=\prn*{ X_1,\dots,X_n }$.
	Then, for any $\delta>0$, with probability of at least $1-\delta$ we have
	\begin{align*}
		\E\brk*{X} - \frac{1}{n} \sum_{i=1}^{n}X_i
		\le \sqrt{\frac{2 V_n\prn*{\bX} \ln\prn*{2/\delta} }{n}} + c \frac{ 14 \ln\prn*{2/\delta} }{ 3 \prn*{ n-1 } }.
	\end{align*}
\end{theorem}
Note that by a union bound \cref{thm: maurer 4} immediately implies that for any $\delta>0$, with probability of at least $1-\delta$ we have
\begin{align*}
	\abs*{\E\brk*{X} - \frac{1}{n} \sum_{i=1}^{n}X_i}
	\le \sqrt{\frac{2 V_n\prn*{\bX} \ln\prn*{4/\delta} }{n}} + c \frac{ 14 \ln\prn*{4/\delta} }{ 3 \prn*{ n-1 } }.
\end{align*}

\begin{theorem}[Theorem 10 of \citet{maurer2009empirical}]
	\label{thm: maurer 10}
	Let $c>0$, $n\ge2$, and $\bX = \prn*{X_1,\dots,X_n}$ be a vector of independent random variables with values in $\brk*{-c,c}$.
	Then, for any $\delta>0$, we have
	\begin{align*}
		\P\brk*{ \sqrt{ \E_{\bX} \brk*{ V_n\prn*{\bX} } } > \sqrt{V_n\prn*{\bX}} + c \sqrt{\frac{8 \ln\prn*{ 1/\delta } }{ n-1 }} } \le &\delta, \quad\text{and}\\
		\P\brk*{ \sqrt{V_n\prn*{\bX}} > \sqrt{ \E_{\bX} \brk*{ V_n\prn*{\bX} } } + c \sqrt{\frac{8 \ln\prn*{ 1/\delta } }{ n-1 }} } \le &\delta.
	\end{align*}
\end{theorem}

\begin{definition}[From \citet{pinelis1994optimum}]
	\label{def: 2-D-smooth}
	A function $\Psi : \X \mapsto \R$ is called $(2,D)$-smooth for some $D>0$ if, for all $x,v\in\X$, we have
	\begin{align*}
		\Psi\prn*{0} &= 0,\\
		\abs{ \Psi\prn*{x+v} - \Psi\prn*{x} } &\le \norm{v}, \quad\qquad\text{and}\\
		\Psi^2\prn*{x+v} - 2\Psi^2\prn*{x} + \Psi^2\prn*{x-v} &\le 2D^2\norm{v}^2.
	\end{align*}
\end{definition}
In particular, for the Euclidean norm, the function $\Psi\prn*{x}=\norm{x}_2$ is $(2,1)$-smooth.

\begin{theorem}[Corollary 10a of \citet{howard2020time}]
	\label{thm: howard 10}
	Consider a martingale $(Y_t)_{t\in \N}$ taking values in a separable Banach space $(\X,\norm{\cdot})$.
	Let the function $\Psi : \X \mapsto \R$ be $(2,D)$-smooth, and define $D_\star\defeq 1\vee D$.
	Suppose $\norm{Y_t-Y_{t-1}} \le c_t$ a.s. for all $t\in\N$ for some constants $(c_t)_{t\in \N}$, and let $V_t\defeq \sum_{i=1}^{t}c_i^2$. Then, for any $x,m > 0$, we have
	\begin{align*}
		\P\brk*{ \exists t\in\N : \Psi\prn*{Y_t} \ge x + \frac{x}{2m} \cdot \prn*{ V_t - m} } \le 2 \exp\prn*{ -\frac{x^2}{2D_\star^2 m} }.
	\end{align*}
\end{theorem}

\subsection{Bounds on regularized solutions}

\begin{lemma}[Lemma 2 of \citet{lawrence2025sample}]
	\label{lem:bound-regularization-direction}
	For any $\theta\ge0$, define
	\begin{align*}
		\hat{x}_{\theta}^{\star} \in \argmin_{x \in \X} \crl*{ \fFo{x} + \theta \norma{x-x_0} }.
	\end{align*}
	If \Cref{asm: grad approx} holds, then
	\begin{align*}
		\P\prn*{  \norma{\hat{x}_{3\lambda}^{\star} - x_0} \le 3 \norma{\hat{x}_{\lambda} - x_0} \le 33 \norma{\hat{x}_{\lambda/3}^{\star} - x_0} } \ge 1 - 2\delta,
	\end{align*}
	where $\lambda = 9\cdot66\phi\frac{\fL\sqrt{\ln\frac{\psi_\alpha}{\delta}}}{\sqrt{N}}$.
\end{lemma}

\section{Proofs of lemmas}

\subsection{Bounding distances between outputs of a Lipschitz function}
\label{sec: m diff Lip}

We now bound $\normb{\mt{x}{s} - \mt{y}{s}}$ under the assumption $m\in\ILip[0][L\opt][\normo{\cdot}]$.
\begin{lemma}
	\label{lem: m diff Lip}
	If $m\in\ILip[0][L\opt][\normo{\cdot}]$, then for all $x,y\in\X$ and $s\in\S$ we have
	\begin{align*}
		\normb{\mt{x}{s} - \mt{y}{s}}
		\le L\opt \norma{x - y}.
	\end{align*}
\end{lemma}

\begin{proof}
	For all $x,y\in\X$ define the line $\theta_{x,y}\prn*{a} \defeq \prn{1-a} \cdot y + a\cdot x$.
	We conclude that for all $x,y\in\X$ and $s\in\S$ we have
	\begin{align*}
		\normb{ \mt{x}{s} - \mt{y}{s} }
		&= \normb{ \int_{y}^{x} \grad_x \mt{x}{s} dx }\\
		&= \normb{ \int_{0}^{1} \grad_x \mt{\theta_{x,y}\prn*{a}}{s} \grad_a\theta_{x,y}\prn*{a} da }\\
		&\overset{(1)}{\le} \int_{0}^{1} \normb{\grad_x \mt{\theta_{x,y}\prn*{a}}{s} \grad_a\theta_{x,y}\prn*{a}} da\\
		&\le \int_{0}^{1} \normo{\grad_x \mt{\theta_{x,y}\prn*{a}}{s}} \norma{\grad_a\theta_{x,y}\prn*{a}} da\\
		&= \norma{ x - y } \int_{0}^{1} \normo{\grad_x \mt{\theta_{x,y}\prn*{a}}{s}} da,
	\end{align*}
	where (1) follows from the triangle inequality for integrals.
	Therefore, since  $m\in\ILip[0][L\opt][\normo{\cdot}]$, we conclude that for all $x,y\in\X$ and $s\in\S$ we have
	\begin{align*}
		\normb{ \mt{x}{s} - \mt{y}{s} }
		\le L\opt \norma{ x - y }.
	\end{align*}
\end{proof}

\subsection{Bounding distances between outputs of a second-moment-Lipschitz function}
\label{sec: m diff sm-Lip}

Analogously to \Cref{sec: m diff Lip}, we derive a similar result for the second-moment Lipschitz case.
\begin{lemma}
	\label{lem: m diff sm-Lip}
	If $m\in\IsLip[0][L\opt][\normo{\cdot}]$ then for all $x,y\in\X$ we have
	\begin{align*}
		\E_{s\sim \PP} \brk*{ \normb{ \mt{x}{s} - \mt{y}{s} }^2 }
		\le L\opt^2 \norma{ x - y }^2.
	\end{align*}
\end{lemma}

\begin{proof}
	For all $x,y\in\X$ define the line $\theta_{x,y}\prn*{a} \defeq \prn{1-a} \cdot y + a\cdot x$.
	Thus, for all $x,y\in\X$ we have
	\begin{align*}
		\E_{s\sim \PP} \brk*{ \normb{ \mt{x}{s} - \mt{y}{s} }^2 }
		&= \E_{s\sim \PP} \brk*{ \normb{ \int_{y}^{x} \grad_x \mt{x}{s} dx }^2 }\\
		&= \E_{s\sim \PP} \brk*{ \normb{ \int_{0}^{1} \grad_x \mt{\theta_{x,y}\prn*{a}}{s} \grad_a\theta_{x,y}\prn*{a} da }^2 }\\
		&\overset{(1)}{\le} \E_{s\sim \PP} \brk*{ \prn*{ \int_{0}^{1} \normb{\grad_x \mt{\theta_{x,y}\prn*{a}}{s} \grad_a\theta_{x,y}\prn*{a}} da }^2 }\\
		&\le \E_{s\sim \PP} \brk*{ \prn*{ \int_{0}^{1} \normo{\grad_x \mt{\theta_{x,y}\prn*{a}}{s}} \norma{\grad_a\theta_{x,y}\prn*{a}} da }^2 }\\
		&\overset{(2)}{\le} \E_{s\sim \PP} \brk*{ \int_{0}^{1} \norma{\grad_a\theta_{x,y}\prn*{a}}^2 da \int_{0}^{1} \normo{\grad_x \mt{\theta_{x,y}\prn*{a}}{s}}^2 da }\\
		&= \norma{ x - y }^2 \int_{0}^{1} \E_{s\sim \PP} \brk*{ \normo{\grad_x \mt{\theta_{x,y}\prn*{a}}{s}}^2 } da,
	\end{align*}
	where (1) follows from the triangle inequality for integrals, and (2) follows from the Cauchy--Schwarz inequality.
	Therefore, because $m\in\IsLip[0][L\opt][\normo{\cdot}]$, for all $x,y\in\X$ we have
	\begin{align*}
		\E_{s\sim \PP} \brk*{ \normb{ \mt{x}{s} - \mt{y}{s} }^2 }
		\le L\opt^2 \norma{ x - y }^2.
	\end{align*}
\end{proof}

\subsection{The risk term $\tau$ fulfills the requirement of \Creflem{lem: reliable selection results}}
\label{sec: assumption hold for alg}

\reslem*

\begin{proof}
	For any $k\in\N$ and $g\in\Grid$, define
	\begin{align*}
		\bar{c}_{g,k} &\defeq 2^k\lell\norma{x_{g}-x_0},\\
		\bar{u}_{g,k} &\defeq \frac{1}{N} \sum_{j=1}^{N} \prn*{ \fl{x_{g}}{\bar{c}_{g,k}}{s_j} - \ft{x_0}{s_j} },\\
		\bar{v}_{g,k} &\defeq \frac{1}{N - 1} \sum_{j=1}^{N} \prn*{ \prn*{\fl{x_{g}}{\bar{c}_{g,k}}{s_j} - \ft{x_0}{s_j} }- \bar{u}_{g,k} }^2,\\
		\tilde{a}_k &\defeq \prn*{k+1}^2\abs{\Grid},\quad \text{and}\\
		\bar{\tau}_{g,k} &\defeq \sqrt{ \frac{ 2 \bar{v}_{g,k} \ln\prn*{16  \tilde{a}_k / \delta} }{ N } } + \bar{c}_{g,k} \frac{14 \ln\prn*{16  \tilde{a}_k / \delta} }{ 3 \prn*{ N-1 } }.
	\end{align*}
	Note that for every $g\in\Grid$ such that $x_g\ne x_0$,
	\begin{align*}
		\tilde{a}_k = \abs{\Grid} \log^2_2 \prn*{\frac{2^{k+1}\lell\norma{x_{g}-x_0}}{\lell\norma{x_{g}-x_0}}}.
	\end{align*}
	Additionally, for every $k\in\N$ and $g\in\Grid$ such that $x_g= x_0$, we have $\bar{\tau}_{g,k}=0$.

	For any $g\in\Grid$, $k\in\N$, and any sample $s\in \S$,
	\begin{align*}
		\abs{\fl{x_{g}}{\bar{c}_{g,k}}{s} - \ft{x_0}{s}}
		\le \normb{\ml{x_{g}}{\bar{c}_{g,k}}{s} - \mt{x_0}{s}}
		\le \bar{c}_{g,k},
	\end{align*}
	where the first inequality is because the loss function is 1-Lipschitz, and the second inequality follows from the definition of the clipped model.
	As a result, by applying Theorem 4 of \citet{maurer2009empirical} (see \Cref{thm: maurer 4}), we obtain that for any $g\in\Grid$ and $k\in\N$, with probability at least $1-\delta/\prn*{4k^2\abs{\Grid} }$
	\begin{align*}
		&\abs*{ \Ft{x_{g}}{\bar{c}_{g,k}} - \Fo{x_0} -\frac{1}{N} \sum_{j=1}^{N} \prn*{ \fl{x_{g}}{\bar{c}_{g,k}}{s_j} - \ft{x_0}{s_j} } }\\
		&\qquad\qquad\qquad\qquad\le \sqrt{ \frac{ 2 \bar{v}_{g,k} \ln\prn*{16  \tilde{a}_k / \delta} }{ N } } + \bar{c}_{g,k} \frac{14 \ln\prn*{16  \tilde{a}_k / \delta} }{ 3 \prn*{ N-1 } }
		= \bar{\tau}_{g,k}.
	\end{align*}
	Thus, by using the union bound and since $\sum_{k\in\N}1/\prn*{k+1}^2 \le 2$, with probability at least $1 - \delta/2$, for any $g\in\Grid$ and $k\in\N$,
	\begin{align*}
		\abs*{ \Ft{x_{g}}{\bar{c}_{g,k}} - \Fo{x_0} -\frac{1}{N} \sum_{j=1}^{N} \prn*{ \fl{x_{g}}{\bar{c}_{g,k}}{s_j} - \ft{x_0}{s_j} } }
		\le \bar{\tau}_{g,k},
	\end{align*}
	Finally, since for all $i\in\brk*{n}$ we have that $g_i\in\Grid$ and $c_i \in \crl*{ \bar{c}_{g_i,k} ~|~ k\in\N }$, the same bound holds for every $c_i$ with probability at least $1 - \delta/2$:
	\begin{align*}
		\abs*{ \Ft{x_{g_i}}{c_{i}} - \Fo{x_0} -\frac{1}{N} \sum_{j=1}^{N} \prn*{ \fl{x_{g_i}}{c_{i}}{s_j} - \ft{x_0}{s_j} } }
		\le \tau_{i},
	\end{align*}
	i.e. the requirement of \Cref{lem: reliable selection results} holds for $\delta/2$.
\end{proof}

\subsection{The sample maximum exceeds a high quantile}

\begin{lemma}
	\label{lem: greater than quantile}
	Let $n\ge1$ and $\delta\in\prn*{0,\frac{1}{e}}$ such that $\ln(1/\delta) < n$.
	Define $p=1-\frac{\ln(1/\delta)}{n}$.
	For any $z\ge n$ and i.i.d. random variables $X_1, \dots, X_z \sim \PP$, we have
	\begin{align*}
		\P\brk*{ \max\prn*{ X_1, \dots, X_z } < \Q{\PP}{p} } < \delta.
	\end{align*}
\end{lemma}

\begin{proof}
	Let $p=1-\frac{\ln(1/\delta)}{n}$.
	For $\max\prn*{ X_1, \dots, X_z } < \Q{\PP}{p}$ to hold, every $X_1, \dots, X_z$ must be less than $\Q{\PP}{p}$.
	The probability of that is less than $p^z$.
	A well-known inequality states that for any $k>0$ we have that $\ln\prn*{ k+1 } - \ln\prn*{ k } = \int_k^{k+1} \frac{1}{t}\,dt \ge \frac{1}{k+1}$.
	Thus,
	\begin{align*}
		\ln\prn*{ \frac{n}{n - \ln(1/\delta)} }
		&= \prn*{\ln\prn*{ n } - \ln\prn*{ n-1 }} + \cdots + \prn*{\ln\prn*{ \ceil*{n - \ln(1/\delta)} } - \ln\prn*{ n - \ln(1/\delta) }}\\
		&\ge \frac{1}{n} + \frac{1}{n-1} + \cdots + \frac{1}{\ceil*{n - \ln(1/\delta)}}\\
		&\ge \frac{\ln(1/\delta)}{n}.
	\end{align*}
	Therefore, since $z\ge n$,
	\begin{align*}
		z \ge \ln\prn*{1/\delta} / \ln\prn*{\frac{n}{n - \ln(1/\delta)} }.
	\end{align*}
	As a consequence,
	\begin{align*}
		z\ln\prn*{\frac{n}{n - \ln(1/\delta)} } \ge \ln\prn*{1/\delta}.
	\end{align*}
	Thus,
	\begin{align*}
		\prn*{1 - \frac{\ln(1/\delta)}{n} }^z \le \delta,
	\end{align*}
	Finally, this means that
	\begin{align*}
		\P\brk*{ \max\prn*{ X_1, \dots, X_z } < \Q{\PP}{p} } < p^z \le \delta.
	\end{align*}
\end{proof}

\subsection{How many samples are needed to obtain $N$ good samples}

\begin{lemma}
	\label{lem: redraw samples}
	Let $a,M,k,N \ge 0$ such that $N - a \cdot e > 0$, and for $X\sim\DD$ we have $\P\brk*{X > M} \le \frac{a}{N}$.
	If we draw at least
	\begin{align*}
		k\cdot N + \max\prn*{1, \frac{\ln\prn*{1/\delta} + a \cdot e \cdot k }{N - a \cdot e} \cdot N}
	\end{align*}
	random variables i.i.d. from $\DD$, then with probability at least $1-\delta$, at least $k\cdot N$ of the drawn variables are at most $M$.
\end{lemma}

\begin{proof}
	We define the random variable $Z_{i}$ as
	\begin{align*}
		Z_{i} = \begin{cases}
			1 & \text{if } X_i > M\\
			0 & \text{otherwise}
		\end{cases}.
	\end{align*}
	We have
	\begin{align}
		\label{eq: Pr Z=1 additional}
		\P\brk*{Z_{i}=1} = \P\brk*{X_i > M} \le \frac{a}{N}.
	\end{align}

	Let $r\ge1$ be the number of additional samples we draw in addition to the $k\cdot N$ we want.
	Using the exponential Chebyshev's inequality (see \Cref{thm: exp Chebyshev}), we obtain, for all $t>0$,
	\begin{align*}
		\P\brk*{\sum_{i=1}^{k\cdot N + r} Z_{i} \ge r}
		&\le e^{-t r} \E\brk*{ e^{t \sum_{i=1}^{k\cdot N + r} Z_{i} } } \\
		&\overset{(1)}{=} e^{-t r} \prod_{i=1}^{k\cdot N + r} \E\brk*{ e^{t Z_{i} } } \\
		&\le e^{-t r} \prod_{i=1}^{k\cdot N + r} \prn*{ 1 \cdot e^{0} + \P\brk*{Z_{i}=1} e^{t} }\\
		&\overset{(2)}{\le} e^{-t r} \prn*{ 1 + \frac{ a \cdot e^{t} }{N}  }^{k\cdot N + r}\\
		&\le e^{-t r + a \cdot e^{t} \prn*{ k \cdot N + r } / N },
	\end{align*}
	where $(1)$ is because $Z_{i}$ are independent of each other, and $(2)$ is by \Cref{eq: Pr Z=1 additional}.
	Setting $t = 1$ yields
	\begin{align*}
		\P\brk*{\sum_{i=1}^{k\cdot N + r} Z_{i} \ge r}
		&\le e^{-r + a \cdot e \prn*{ k + r/N } }.
	\end{align*}
	We obtain
	\begin{align*}
		e^{-r + a \cdot e \prn*{ k + r/N } } \le \delta
	\end{align*}
	if and only if
	\begin{align*}
		r \prn*{1 - a \cdot e / N } \ge \ln\prn*{1/\delta} + a \cdot e \cdot k.
	\end{align*}
	Therefore, if $a < N / e$ and
	\begin{align*}
		r \ge \max\prn*{1, \frac{\ln\prn*{1/\delta} + a \cdot e \cdot k }{N - a \cdot e}\cdot N},
	\end{align*}
	we obtain
	\begin{align*}
		\P\brk*{\sum_{i=1}^{k\cdot N + r} Z_{i} \ge r}
		\le \delta.
	\end{align*}
	This implies that, with probability at least $1 - \delta$, fewer than $r$ samples exceed $M$, and hence at least $k \cdot N$ samples are at most $M$.
\end{proof}

\subsection{Empirical mean approximates the expectation}

\begin{lemma}
	\label{lem: bound on average}
	Let $\norm{\cdot}$ be a $\prn*{2,D}$-smooth (\Cref{def: 2-D-smooth}) norm for some $D\ge1$.
	Let $V_1, \dots, V_n$ be a sequence of i.i.d. random vectors in $\R^d$. Define $v\defeq\E[V_i]$.
	If $\norm{V_i - v} \le C$ almost surely for some constant $C$, then for all $\delta\in\prn*{0,\frac{2}{e}}$,
	\begin{align*}
		\P\brk*{ \norm*{\frac{1}{n} \sum_{i=1}^{n} V_i- v} \ge \frac{ D C \sqrt{2\ln \prn*{2/\delta}} }{\sqrt{n}} } \le \delta.
	\end{align*}
\end{lemma}

\begin{proof}
	For any $k\in\N$ define $Y_k\defeq \frac{1}{n} \sum_{i=1}^{k} \prn*{V_i - v}$.
	Denote $Y_0=0$ and $\Psi\prn*{\cdot}\defeq\norm{\cdot}$.
	We have $\norm{Y_i - Y_{i-1}} \le C/n$ for all $i\in\N$, and $\crl*{Y_i}_{i\in\N}$ is a martingale.
	Thus, from Corollary 10a of \citet{howard2020time} (see \Cref{thm: howard 10}), for all $x,z>0$,
	\begin{align*}
		\P\brk*{ \exists t\in\N : \Psi\prn*{Y_t} \ge x + \frac{D^2 x}{2z} \cdot \prn*{ \frac{tC^2}{n^2} - z }} \le 2 \exp\prn*{ -\frac{x^2}{2D^2 z} }.
	\end{align*}
	Therefore, for $x=D C \sqrt{2\ln \prn*{2/\delta}} / \sqrt{n}$ and $z=C^2/n$, we have
	\begin{align*}
		\P\brk*{ \norm*{\frac{1}{n} \sum_{i=1}^{n} V_i - v} \ge \frac{ D C \sqrt{2 \ln \prn*{2/\delta}} }{\sqrt{n}} } \le 2 \exp\prn*{ -\ln \prn*{2/\delta} } = \delta.
	\end{align*}
\end{proof}

\section{The mean of stochastic gradients approximates the true gradient}
\label{sec: stochastic gradients approx gradient}

\begin{lemma}
	\label{lem: stochastic gradients approx gradient}
	Let $d\in\N$ and let $\X\subseteq\R^d$.
	\begin{enumerate}
		\item In the case that $\norma{\cdot} = \norm{\cdot}_2$, \Cref{asm: grad approx} holds for $\psi_\alpha=1$ and any $\phi\ge1$.
		\item In the case that $\norma{\cdot} = \norm{\cdot}_1$, \Cref{asm: grad approx} holds for $\psi_\alpha=d$ and any $\phi\ge1$.
	\end{enumerate}
\end{lemma}

We first prove \Cref{lem: stochastic gradients approx gradient} in the case that $\norma{\cdot} = \norm{\cdot}_2$.
\begin{proof}
	In the case that $\norma{\cdot} = \norm{\cdot}_2$ then we also have that $\normas{\cdot} = \norm{\cdot}_2$.
	We note that $\norm{\cdot}_2$ is $\prn*{2,1}$-smooth.
	Therefore, as $f\prn*{\cdot,s}$ is $\fL$-Lipschitz for every $s$ sampled from $\tilde{\PP}$, a direct result of \Cref{lem: bound on average} is that
	\begin{align*}
		\mathop{\P}_{s_1,\dots,s_N\iid\tilde{\PP}}\brk*{ \norm*{\nabla \E_{s\sim\tilde{\PP}} \ft{x}{s} - \frac{1}{N} \sum_{i=1}^{N} \nabla \ft{x}{s_i}}_2 \ge \frac{ \fL \sqrt{8\ln \prn*{2/\delta}} }{\sqrt{N}} } \le \delta.
	\end{align*}
	Thus, \Cref{asm: grad approx} holds for $\psi_\alpha=1$ and any $\phi\ge1$.
\end{proof}

We now prove \Cref{lem: stochastic gradients approx gradient} in the case that $\norma{\cdot} = \norm{\cdot}_1$.
\begin{proof}
	In the case that $\norma{\cdot} = \norm{\cdot}_1$ then we also have that $\normas{\cdot} = \norm{\cdot}_\infty$.
	Therefore, following the result of the case that $\norma{\cdot} = \norm{\cdot}_2$, for every coordinate $j\in\brk*{d}$,
	\begin{align*}
		\mathop{\P}_{s_1,\dots,s_N\iid\tilde{\PP}}\brk*{ \abs*{ \brk*{ \nabla \E_{s\sim\tilde{\PP}}\ft{x}{s} - \frac{1}{N} \sum_{i=1}^{N} \nabla \ft{x}{s_i} }_j } \ge \frac{ \fL \sqrt{8\ln \prn*{2d/\delta}} }{\sqrt{N}} } \le \frac{\delta}{d}.
	\end{align*}
	Thus, by a union bound we obtain that
	\begin{align*}
		\mathop{\P}_{s_1,\dots,s_N\iid\tilde{\PP}}\brk*{ \norm*{ \nabla \E_{s\sim\tilde{\PP}} \ft{x}{s} - \frac{1}{N} \sum_{i=1}^{N} \nabla \ft{x}{s_i} }_\infty \ge \frac{ \fL \sqrt{8\ln \prn*{2d/\delta}} }{\sqrt{N}} } \le \delta.
	\end{align*}
	As a consequence, \Cref{asm: grad approx} holds for $\psi_\alpha=d$ and any $\phi\ge1$.
\end{proof}

\section{Proof of \Crefpro{thm: clipped model subopt} for the case that $m$ is second-moment-Lipschitz}
\label{sec: sm-lip selection theorem}

The following theorem is a slight generalization of \Cref{thm: clipped model subopt}; the choice $\hat{a}=1$ recovers \Cref{thm: clipped model subopt}.
\begin{proposition}
	Let $m\in\IsLip[\lell][\infty][\normo{\cdot}]$ and $h\in\ILip[0][1][\normbs{\cdot}]$.
	Let $\hat{a}\in\brk*{1, \abs{\Grid} \log_2^2\prn*{2\sqrt{N} \cdot \uell/\lell}}$, where $\uell$ is an upper bound on the Lipschitz constant.
	Let the number of samples be $N\ge 4 \ln \prn*{ 4\hat{a}/\delta }$.
	If $x,c$ are the result of \Cref{alg: clipping around initial}, where we use $n=\ceil{4 \ln \prn*{ 4\hat{a}/\delta }}$, then with probability at least $1-\delta$ we have
	\begin{align*}
		\Ft{x}{c} \le \min_{g \in \Grid\cup\crl*{0}} \brk* { \Fo{x_g} + \O{ \Lopt \norma{x_g - x_0} \sqrt{\frac{\lnp\prn*{ \tilde{a} / \delta} }{N} } \sqrt{\frac{\lnp\prn*{ \tilde{a} / \delta} }{ \lnp\prn*{ \hat{a} / \delta} } } } },
	\end{align*}
	where $\tilde{a} = \max\prn*{\hat{a}, \abs{\Grid} \lnp\prn*{\sqrt{N} \Lopt/\lell}}$.
\end{proposition}

\begin{proof}
	Let $\hat{a}\in\brk*{1, \abs{\Grid} \log_2^2\prn*{2\sqrt{N} \cdot \uell/\lell}}$.
	Let $n$ in \Cref{alg: clipping around initial} be $n=4 \ln \prn*{ 4\hat{a}/\delta}$.
	Finally, define $\tilde{a} = \max\prn*{\hat{a}, \abs{\Grid} \log_2^2 \prn*{2 \sqrt{N} \cdot \Lopt/\lell}}$.

	First, as $m\in\IsLip[\lell][L\opt][\normo{\cdot}]$, \Cref{lem: m diff sm-Lip} states that for all $g\in\Grid$,
	\begin{align}
		\label{eq: sm-Lipschitz of function}
		\E_{s\sim \PP} \brk*{ \normb{ \mt{x_{g}}{s} - \mt{x_{0}}{s} }^2 }
		\le L\opt^2 \norma{ x_{g} - x_{0} }^2.
	\end{align}

	Let
	\begin{align}
		\label{eq: gopt def}
		g\opt = \argmin_{g \in \Grid\cup\crl*{0}} \Fo{x_g} + 94 \Lopt \norma{ x_{g} - x_{0} } \sqrt{\frac{\ln\prn*{ 16\tilde{a} / \delta}}{N}} \sqrt{ \frac{\ln\prn*{16  \tilde{a} / \delta}}{\ln\prn*{4  \hat{a} / \delta}} }.
	\end{align}
	Define $c_{g\opt,0}\defeq \al[\sqrt{\frac{N}{\ln \prn*{ 4\hat{a}/\delta }}} L\opt \norma{x_{g\opt} - x_0}][\lell \norma{x_{g\opt} - x_0}]$.
	For candidate $\prn*{x_{g\opt}, c_{g\opt,0}}$ we define $\tau_{g\opt,0}$ later in the proof.
	In the proof we use $\crl*{\prn*{ x_g, c_{g,i} }}_{\prn*{g,i}\in \Grid \times \brk*{n}}\cup\crl*{\prn*{x_{g\opt}, c_{g\opt,0}}}$ as the set of candidates passed to \Crefn{alg: reliable model selection}.
	We show that, with high probability, we can assume that \Cref{alg: reliable model selection} returns the same output under both the original set of candidates and the modified set of candidates.

	We define the random variable $Z_{i}$ as
	\begin{align*}
		Z_{i} = \begin{cases}
			1 & \text{if } \al[\normb{ \mt{x_{g\opt}}{s_i} - \mt{x_0}{s_i} }][\lell \norma{x_{g\opt} - x_0}] > c_{g\opt,0}\\
			0 & \text{otherwise}
		\end{cases}.
	\end{align*}
	We now find an upper bound on $\P\brk*{Z_{i}=1}$.
	First, we bound the variance $\Var_{s\sim \PP}\brk*{\normb{\mt{x_{g\opt}}{s} - \mt{x_0}{s}}}$:
	\begin{align*}
		\Var_{s\sim \PP}\brk*{\normb{\mt{x_{g\opt}}{s} - \mt{x_0}{s}}}
		\le \E_{s\sim \PP}\brk*{\normb{\mt{x_{g\opt}}{s} - \mt{x_0}{s}}^2}
		\le \prn*{L\opt \norma{ x_{g\opt} - x_{0} } }^2,
	\end{align*}
	where the last inequality is the result of \Cref{eq: sm-Lipschitz of function}.
	Therefore,
	\begin{align*}
		\label{eq: Pr Z=1}
		\P\brk*{Z_{i}=1}
		&= \P\brk*{ \al[\normb{\mt{x_{g\opt}}{s_i} - \mt{x_0}{s_i}}][\lell \norma{x_{g\opt} - x_0}] > c_{g\opt,0}}\\
		&\le \P\brk*{\normb{\mt{x_{g\opt}}{s_i} - \mt{x_0}{s_i}}>\sqrt{\frac{N}{\ln \prn*{ 4\hat{a}/\delta }}} L\opt \norma{ x_{g\opt} - x_{0} }}\\
		&\le \frac{ \ln \prn*{ 4\hat{a}/\delta } }{N}, \numberthis
	\end{align*}
	where the last inequality is a result of Chebyshev's inequality (see \Cref{thm: Chebyshev}).

	As a result, using the exponential Chebyshev's inequality (see \Cref{thm: exp Chebyshev}), we obtain, for all $t>0$ and $r>0$,
	\begin{align*}
		\P\brk*{\sum_{i=1}^{N} Z_{i} \ge r}
		&\le e^{-t r} \E\brk*{ e^{t \sum_{i=1}^{N} Z_{i} } } \\
		&\overset{(1)}{=} e^{-t r} \prod_{i=1}^{N} \E\brk*{ e^{t Z_{i} } } \\
		&\le e^{-t r} \prod_{i=1}^{N} \prn*{ 1 \cdot e^{0} + \P\brk*{Z_{i}=1} e^{t} }\\
		&\overset{(2)}{\le} e^{-t r} \prn*{ 1 + \frac{ e^{t} \cdot \ln \prn*{ 4\hat{a}/\delta } }{N}  }^{N}\\
		&\le e^{-t r + e^{t} \cdot \ln \prn*{ 4\hat{a}/\delta } },
	\end{align*}
	where $(1)$ is because $Z_{i}$ are independent of each other, and $(2)$ is by \Cref{eq: Pr Z=1}.
	Thus, by setting $t=1$ and $r=4 \ln \prn*{ 4\hat{a}/\delta}$ we obtain that
	\begin{align*}
		\P\brk*{\sum_{i=1}^{N} Z_{i} \ge 4 \ln \prn*{ 4\hat{a}/\delta}} \le e^{-\ln \prn*{ 4\hat{a}/\delta}} \le \frac{\delta}{4\hat{a}} \le \frac{\delta}{4}.
	\end{align*}
	This implies that with probability at least $1-\frac{\delta}{4}$ there are fewer than $4 \ln \prn*{ 4\hat{a}/\delta}$ samples $s_i$ for which $\mt{x_{g\opt}}{s_i}$ is clipped when $c_{g\opt,0}$ is used as the clipping threshold.
	As a result, with probability at least $1-\frac{\delta}{4}$, there exists $i\in\brk*{n}$ for which $c_{g\opt,0} \ge c_{g\opt,i}$.
	Let $i\opt=\argmax_{i\in\brk*{n};c_{g\opt,i}\le c_{g\opt,0}} c_{g\opt,i}$.
	The validation loss for both clipping parameters $c_{g\opt,0}$ and $c_{g\opt,i}$ is very close:
	\begin{align*}
		\abs*{ \frac{1}{N} \sum_{j=1}^{N} \fl{x_{g\opt}}{c_{g\opt,0}}{s_j} - \frac{1}{N} \sum_{j=1}^{N} \fl{x_{g\opt}}{c_{g\opt,i\opt}}{s_j} }
		&\le \frac{1}{N} \sum_{j=1}^{N} \normb{ \ml{x_{g\opt}}{c_{g\opt,0}}{s_{j}} - \ml{x_{g\opt}}{c_{g\opt,i\opt}}{s_{j}} }\\
		&\le 8 c_{g\opt,0} \frac{ \ln \prn*{ 4\hat{a}/\delta} }{ N },
	\end{align*}
	where the first inequality is because $\lt{\cdot}{ s}$ is 1-Lipschitz.
	Thus, from the definition of $c_{g\opt,0}$, we get
	\begin{align}
		\label{eq: validation almost same}
		\abs*{ \frac{1}{N} \sum_{j=1}^{N} \fl{x_{g\opt}}{c_{g\opt,0}}{s_j} - \frac{1}{N} \sum_{j=1}^{N} \fl{x_{g\opt}}{c_{g\opt,i\opt}}{s_j} }
		&\le  16 \Lopt \norma{x_{g\opt} - x_0} \sqrt{\frac{ \ln \prn*{ 4\hat{a}/\delta} }{ N }}.
	\end{align}

	For candidate $\prn*{x_{g\opt}, c_{g\opt,0}}$, we define $\tau_{g\opt,0}$ as follows:
	\begin{align*}
		\tau_{g\opt,0} \defeq
		\max\prn*{ \te{ x_{g\opt} }{ c_{g\opt,0} }, \tau_{g\opt,i\opt} + 16 \Lopt \norma{x_{g\opt} - x_0} \sqrt{\frac{ \ln \prn*{ 4\hat{a}/\delta} }{ N }} }.
	\end{align*}
	Therefore, from \Cref{eq: validation almost same},
	\begin{align*}
		\frac{1}{N} \sum_{j=1}^{N} \fl{x_{g\opt}}{c_{g\opt,i\opt}}{s_j} + \tau_{g\opt,i\opt} \le \frac{1}{N} \sum_{j=1}^{N} \fl{x_{g\opt}}{c_{g\opt,0}}{s_j} + \tau_{g\opt,0}.
	\end{align*}
	As a consequence, (i) the $\theta$ calculated in \Cref{alg: reliable model selection} is not affected by the candidate $\prn*{x_{g\opt}, c_{g\opt,0}}$, and (ii) for $\F$ calculated by \Cref{alg: reliable model selection}, if $\prn*{x_{g\opt}, c_{g\opt,0}}\in\F$ then also $\prn*{x_{g\opt}, c_{g\opt,i\opt}}\in\F$.
	Therefore, by using \Cref{lem: reliable selection results gen} instead of \Cref{lem: reliable selection results}, we can assume that if \Cref{alg: reliable model selection} should have returned the candidate $\prn*{x_{g\opt}, c_{g\opt,0}}$ it instead returned the candidate $\prn*{x_{g\opt}, c_{g\opt,i\opt}}$.
	This means that with probability at least $1-\frac{\delta}{4}$ \Cref{alg: reliable model selection} is not affected by the candidate $\prn*{x_{g\opt}, c_{g\opt,0}}$, i.e. we can assume that \Cref{alg: reliable model selection} returns the same output for both the original set of candidates and the modified set of candidates.

	We now move to show that $\prn*{x_{g\opt}, c_{g\opt,0}}$ is a good candidate.
	For $g\opt$ and $c_{g\opt,0}$, we cannot show that $\Ft{x_{g\opt}}{c_{g\opt,0}} = \Fo{x_{g\opt}}$.
	Therefore, we bound $\abs{ \Ft{x_{g\opt}}{c_{g\opt,0}} - \Fo{x_{g\opt}} }$ instead.
	For any $c\ge0$, $x\in\X$ and $s\in\S$ such that $\normb{ \mt{x}{s} - \mt{x_{0}}{s} } > c$, we have
	\begin{align*}
		\normb{ \ml{x}{c}{s} - \mt{x}{s} }
		&\le \normb{ \mt{x}{s} - \mt{x_0}{s} } + \normb{ \ml{x}{c}{s} - \mt{x_0}{s} }\\
		&\le \normb{ \mt{x}{s} - \mt{x_0}{s} } + c\\
		&\le 2\normb{ \mt{x}{s} - \mt{x_0}{s} }.
	\end{align*}
	As a consequence, if $c_{g\opt,0}>0$ then
	\begin{align*}
		\abs{ \Ft{x_{g\opt}}{c_{g\opt,0}} - \Fo{x_{g\opt}} }
		&= \E_{s\sim \PP} \brk*{ \abs{ \fl{x_{g\opt}}{c_{g\opt,0}}{s} - \ft{x_{g\opt}}{s} } }\\
		&\le \E_{s\sim \PP} \brk*{ \normb{ \ml{x_{g\opt}}{c_{g\opt,0}}{s} - \mt{x_{g\opt}}{s} } }\\
		&= \E_{s\sim \PP} \brk*{ \normb{ \ml{x_{g\opt}}{c_{g\opt,0}}{s} - \mt{x_{g\opt}}{s} } \indic{ \normb{ \mt{x_{g\opt}}{s} - \mt{x_{0}}{s} } > c_{g\opt,0} }}\\
		&\le 2\E_{s\sim \PP} \brk*{ \normb{ \mt{x_{g\opt}}{s} - \mt{x_{0}}{s} } \indic{ \normb{ \mt{x_{g\opt}}{s} - \mt{x_{0}}{s} } > c_{g\opt,0} }}\\
		&\le \frac{2}{c_{g\opt,0}}\E_{s\sim \PP} \brk*{ \normb{ \mt{x_{g\opt}}{s} - \mt{x_{0}}{s} }^2 },\numberthis
		\label{eq: F almost equiv sm-lip partial}
	\end{align*}
	where the first inequality is because $\lt{\cdot}{s}$ is 1-Lipschitz.
	Additionally, if $c_{g\opt,0}=0$ then $x_{g\opt}=x_{0}$, and thus
	\begin{align*}
		\abs{ \Ft{x_{g\opt}}{c_{g\opt,0}} - \Fo{x_{g\opt}} }
		=\abs{ \Ft{x_0}{c_{g\opt,0}} - \Fo{x_0} }
		=0
	\end{align*}
	Therefore, by combining \Cref{eq: sm-Lipschitz of function,eq: F almost equiv sm-lip partial} and the previous equation, we get
	\begin{align}
		\label{eq: F almost equiv sm-lip}
		\abs{ \Ft{x_{g\opt}}{c_{g\opt,0}} - \Fo{x_{g\opt}} } \le 2 L\opt \norma{ x_{g\opt} - x_{0} } \sqrt{\frac{\ln \prn*{ 4\hat{a}/\delta }}{N}}.
	\end{align}
	Thus, the clipped function is, in expectation, almost equivalent to the non-clipped function.

	We now want to bound $\tau_{g\opt,0}$, similarly to what we have done for the Lipschitz case.
	To do this, we need to bound the sample variance $\ve{g\opt}{c_{g\opt,0}}$.
	We first start by bounding the variance, for any $c\ge0$
	\begin{align*}
		\E_{s\sim \PP} \brk*{ \abs{ \fl{x_{g\opt}}{c}{s} - \ft{x_{0}}{s} }^2 }
		&\le \E_{s\sim \PP} \brk*{ \normb{ \ml{x_{g\opt}}{c}{s} - \mt{x_{0}}{s} }^2 }\\
		&\le \E_{s\sim \PP} \brk*{ \normb{ \mt{x_{g\opt}}{s} - \mt{x_{0}}{s} }^2 }\\
		&\le  L\opt^2 \norma{ x_{g\opt} - x_{0} }^2,
	\end{align*}
	where the first inequality is because $\lt{\cdot}{s}$ is 1-Lipschitz, and the final inequality is from using \Cref{eq: sm-Lipschitz of function} again.
	Now, using Theorem 10 of \citet{maurer2009empirical} (see \Cref{thm: maurer 10}), with probability at least $1-\delta/4$, the sample variance is bounded by
	\begin{align*}
		\sqrt{\ve{g\opt}{c_{g\opt,0}}} &\le L\opt \norma{ x_{g\opt} - x_{0} } +  c_{g\opt,0} \sqrt{\frac{8 \ln\prn*{4/\delta}}{N - 1}}, \qquad \text{and}\\
		\sqrt{\ve{g\opt}{c_{g\opt,i\opt}}} &\le L\opt \norma{ x_{g\opt} - x_{0} } +  c_{g\opt,i\opt} \sqrt{\frac{8 \ln\prn*{4/\delta}}{N - 1}}
		\le L\opt \norma{ x_{g\opt} - x_{0} } +  c_{g\opt,0} \sqrt{\frac{8 \ln\prn*{4/\delta}}{N - 1}}.
	\end{align*}
	Therefore,
	\begin{align*}
		\tau_{g\opt,0} &=
		\max\prn*{ \te{ x_{g\opt} }{ c_{g\opt,0} }, \tau_{g\opt,i\opt} + 16 \Lopt \norma{x_{g\opt} - x_0} \sqrt{\frac{ \ln \prn*{ 4\hat{a}/\delta} }{ N }} }\\
		&\le \prn{16+2} \Lopt \norma{ x_{g\opt} - x_{0} } \sqrt{ \frac{ \ln\prn*{16  \tilde{a} / \delta} }{ N } } + c_{g\opt,0} \sqrt{\frac{8 \ln\prn*{4/\delta}}{N - 1}} \sqrt{ \frac{ 2 \ln\prn*{16  \tilde{a} / \delta} }{ N } } +
		c_{g\opt,0} \frac{14 \ln\prn*{16  \tilde{a} / \delta} }{ 3 \prn*{ N - 1 } }\\
		&\le 18 \Lopt \norma{ x_{g\opt} - x_{0} } \sqrt{ \frac{ \ln\prn*{16  \tilde{a} / \delta} }{ N } } +
		8 \Lopt \norma{x_{g\opt} - x_0} \sqrt{\frac{ \ln\prn*{16  \tilde{a} / \delta} }{N}} \sqrt{\frac{N}{N - 1}} \\
		&\qquad\qquad+ 10 \Lopt \norma{x_{g\opt} - x_0} \sqrt{\frac{ \ln\prn*{16  \tilde{a} / \delta} }{ N }} \sqrt{ \frac{\ln\prn*{16  \tilde{a} / \delta}}{\ln\prn*{4  \hat{a} / \delta}} } \frac{N}{N - 1}.
	\end{align*}
	As a result, since $N\ge2$,
	\begin{align}
		\label{eq: tau sm-lip}
		\tau_{g\opt,0} \le 54 \Lopt \norma{ x_{g\opt} - x_{0} } \sqrt{\frac{\ln\prn*{16 \tilde{a} / \delta}}{N}} \sqrt{ \frac{\ln\prn*{16  \tilde{a} / \delta}}{\ln\prn*{4  \hat{a} / \delta}} }.
	\end{align}

	Thus, by using \Cref{eq: F almost equiv sm-lip,eq: tau sm-lip},
	\begin{align*}
		\label{eq: correct g loss bound SM}
		\Ft{x_{g\opt}}{c_{g\opt,0}} + 2 \tau_{g\opt,0}
		&\le \Fo{x_{g\opt}} + 2 \cdot 55 \Lopt \norma{ x_{g\opt} - x_{0} } \sqrt{\frac{\ln\prn*{ 16\tilde{a} / \delta}}{N}} \sqrt{ \frac{\ln\prn*{16  \tilde{a} / \delta}}{\ln\prn*{4  \hat{a} / \delta}} }.\numberthis
	\end{align*}
	Finally, as \Cref{lem: condition hold for algorithm} shows that the requirement of \Cref{lem: reliable selection results} holds for $\delta/2$, \Cref{lem: reliable selection results,lem: reliable selection results gen} guarantee that if $x,c$ is the result of \Cref{alg: clipping around initial}, then with probability at least $1-\delta/2$,
	\begin{align*}
		\Ft{x}{c}
		\le \min_{\substack{g\in \Grid\cup\crl*{0}\\i \in \brk*{n} \cup \crl*{0}}} \brk*{ \Ft{x_{g}}{c_{g,i}} + 2\tau_{g,i} }.
	\end{align*}
	Now, because with probability at least $1-\delta/4$ the additional candidate $\prn*{g\opt, c_{g\opt,0}}$ does not affect the algorithm, with probability at least $1-\delta$ we have
	\begin{align*}
		\Ft{x}{c}
		\le \min_{\substack{g\in \Grid\cup\crl*{0}\\i \in \brk*{n}}} \brk*{ \Ft{x_{g}}{c_{g,i}} + 2 \tau_{g,i}}.
	\end{align*}
	Hence, by the definition of $g\opt$ in \Cref{eq: gopt def} and by \Cref{eq: correct g loss bound SM}, with probability at least $1-\delta$,
	\begin{align*}
		\Ft{x}{c}
		\le \min_{g \in \Grid\cup\crl*{0}} \brk*{ \Fo{x_g} + \O{ \Lopt \norma{x_g - x_0} \sqrt{\frac{\lnp\prn*{ \tilde{a} / \delta} }{N} } \sqrt{ \frac{\lnp\prn*{\tilde{a} / \delta}}{\lnp\prn*{\hat{a} / \delta}} } } }.
	\end{align*}
\end{proof}

\section{Proof of \Crefthm{thm: sample algorithm}}
\label{sec: sample algorithm proof}

\thmsamplealg*

\begin{proof}
	First, we consider the case where $N\le7$ or $\ln\prn*{1/\delta} > \frac{N}{2e}$.
	In both cases, we have $\ln\prn*{\psi_\alpha/\delta} > \frac{N}{2e}$;
	if $N\le7$, this follows from $\psi_\alpha\ge1$ and $\delta<\frac{1}{5}$;
	otherwise it follows directly from $\psi_\alpha\ge1$.
	Therefore,
	\begin{align*}
		\Ft{x_0}{0}
		= \Fo{x_0}
		\le \Fo{\xopt} + L\opt\norma{\xopt - x_0}
		\le \Fo{\xopt} + \O{\frac{L\opt\norma{\xopt - x_0} \sqrt{\ln\frac{\psi_\alpha}{\delta}} }{\sqrt{N}}}.
	\end{align*}
	In these cases, the number of samples used is 0.

	Now, we move to prove the theorem for the case that $\ln\prn*{\psi_\alpha/\delta} \le \frac{N}{2e}$.
	Define
	\begin{align*}
		\fFo{x} &\defeq \E_{\substack{s\sim\PP}} \brk*{ \ft{x}{s} ~|~ L\prn*{s} \le \fL }.
	\end{align*}
	For any $\theta\ge0$ define
	\begin{align*}
		\hat{x}_{\theta} \in \argmin_{x\in\X} \crl*{ \frac{1}{N}\sum_{i=N+1}^{2N} \ft{x}{s_i} + \theta\norma{x-x_0} } \text{ and } \hat{x}_{\theta}^{\star} \in \argmin_{x \in \X} \crl*{ \fFo{x} + \theta \norma{x-x_0} }.
	\end{align*}
	Also, define $\hat{R} \defeq 3 \norma{\hat{x}_{\lambda} - x_0}$.
	As we assume that \Cref{asm: grad approx} holds, the conditions of Lemma 2 of \citet{lawrence2025sample} are satisfied.
	Thus, from Lemma 2 of \citet{lawrence2025sample} (see \Cref{lem:bound-regularization-direction}), with probability at least $1-2\delta$,
	\begin{align}
		\label{eq: x lambda bound}
		\norma{\hat{x}_{3\lambda}^{\star} - x_0} \le 3 \norma{\hat{x}_{\lambda} - x_0} \le 33 \norma{\hat{x}_{\lambda/3}^{\star} - x_0}.
	\end{align}

	If $\norma{\hat{x}_{\lambda/3}^{\star} - x_0} \ge \hat{R}$ then, because of \Cref{eq: x lambda bound}, there exists $\lambda_{\hat{R}}$ such that $\lambda_{\hat{R}} \in \brk*{ \lambda/3, 3\lambda }$ and $\norma{\hat{x}_{\lambda_{\hat{R}}}^{\star} - x_0} = \hat{R}$ and
	\begin{align}
		\label{eq: R case 1}
		\min_{x\in\X : \norma{x - x_0}\le\hat{R}} \fFo{x} \le \fFo{\hat{x}_{\lambda_{\hat{R}}}^{\star}} \le  \fFo{\xopt} + \lambda_{\hat{R}} \prn*{ \norma{\xopt - x_0} - \norma{\hat{x}_{\lambda_{\hat{R}}}^{\star} - x_0} }.
	\end{align}
	Otherwise, because of \Cref{eq: x lambda bound}, we have that $\norma{\hat{x}_{\lambda/3}^{\star} - x_0} \in \brk*{ \frac{\hat{R}}{33}, \hat{R} }$, and therefore
	\begin{align}
		\label{eq: R case 2}
		\min_{x\in\X : \norma{x - x_0}\le\hat{R}} \fFo{x} \le \fFo{\hat{x}_{\lambda/3}^{\star}} \le  \fFo{\xopt} + \frac{\lambda}{3} \prn*{ \norma{\xopt - x_0} - \norma{\hat{x}_{\lambda/3}^{\star} - x_0} }.
	\end{align}

	If $\hat{R} \le 66\norma{\xopt - x_0}$ then \Cref{eq: R case 1,eq: R case 2} yields
	\begin{align*}
		\min_{x\in\X : \norma{x - x_0}\le\hat{R}} \fFo{x} \le \fFo{\xopt} + 3\lambda \norma{\xopt - x_0}
	\end{align*}
	Therefore, by \Cref{asm: alg subopt}, with probability at least $1-3\delta$,
	\begin{align*}
		\fFo{\xp}
		&\le \fFo{\xopt} + 3\lambda \norma{\xopt - x_0} + \phi\frac{\fL\hat{R}\sqrt{\ln\frac{\psi_\alpha}{\delta}}}{\sqrt{N}}\\
		&\le \fFo{\xopt} + \O{ \phi\frac{\fL\norma{\xopt - x_0}\sqrt{\ln\frac{\psi_\alpha}{\delta}}}{\sqrt{N}} }. \numberthis
		\label{eq: min fake F 1}
	\end{align*}

	Otherwise, if $\hat{R} \ge 66 \norma{\xopt - x_0}$, we obtain from \Cref{eq: R case 1,eq: R case 2} that
	\begin{align*}
		\min_{x\in\X : \norma{x - x_0}\le\hat{R}} \fFo{x} \le \fFo{\xopt} -  \frac{\lambda\hat{R}}{3 \cdot 66}.
	\end{align*}
	Therefore, by \Cref{asm: alg subopt}, with probability at least $1-3\delta$,
	\begin{align*}
		\fFo{\xp}
		&\le \fFo{\xopt}  -  \frac{\lambda\hat{R}}{3 \cdot 66} + \phi\frac{\fL\hat{R}\sqrt{\ln\frac{\psi_\alpha}{\delta}}}{\sqrt{N}}\\
		&\le \fFo{\xopt} - 2\frac{\fL\hat{R}\sqrt{\ln\frac{\psi_\alpha}{\delta}}}{\sqrt{N}}. \numberthis
		\label{eq: min fake F 2}
	\end{align*}

	Now that we have bounded $\fFo{\xp}$, we use this to bound $\Ft{\xp}{\fL\norma{\xp-x_0}}$.
	Define $L\prn*{s}$ as the Lipschitz constant of $\mt{\cdot}{s}$, and $L\prn*{\PP}$ as the distribution over the Lipschitz constants resulting from the distribution $\PP$.
	\Cref{lem: greater than quantile} states that with probability at least $1-\delta$,
	\begin{align}
		\label{eq: fake L quantile}
		\fL = \max_{i\in\brk*{N}} L\prn*{s_i} \ge \Q{L\prn*{\PP}}{1-\frac{\ln(1/\delta)}{N}}.
	\end{align}
	Thus, for $p\defeq\P_{s\sim\PP}\brk*{ L\prn*{s} > \fL}$, we have $p\in\brk*{0,\frac{\ln(1/\delta)}{N}}$.
	Therefore,
	\begin{align*}
		&\Ft{ \xp }{ \fL \norma{\xp - x_0} }\\
		&\qquad= \prn*{1-p}\E_{\substack{s\sim\PP}}\brk*{\ft{\xp}{s} ~|~ L\prn*{s} \le \fL}\\
		&\qquad\phantom{\mathllap{a}\le \prn*{1-p}}\mathllap{+p}\E_{\substack{s\sim\PP}}\brk*{\lt{\ml{\xp}{\fL\norma{\xp-x_0}}{s}}{s} ~|~ L\prn*{s} > \fL}\\
		&\qquad\le \prn*{1-p}\E_{\substack{s\sim\PP}}\brk*{\ft{\xp}{s} ~|~ L\prn*{s} \le \fL}\\
		&\qquad\phantom{\mathllap{a}\le \prn*{1-p}}\mathllap{+p}\E_{\substack{s\sim\PP}}\brk*{\ft{\xopt}{s} + \fL \norma{\xp - x_0} + L\opt \norma{\xopt - x_0} ~|~ L\prn*{s} > \fL}\\
		&\qquad= \Fo{\xopt} + \prn*{1-p}\prn*{\fFo{\xp} - \fFo{\xopt}}
		+ p\prn*{ \fL \norma{\xp - x_0} + L\opt \norma{\xopt - x_0} }, \numberthis
		\label{eq: F perfect bound}
	\end{align*}
	where the inequality is because $h$ is $1$-Lipschitz and $m$ is $L\opt$-Lipschitz.

	Now we continue to bound $\Ft{\xp}{\fL\norma{\xp-x_0}}$.
	First, let us consider the case that $\hat{R} \le 66\norma{\xopt - x_0}$.
	We have
	\begin{align*}
		&\Ft{ \xp }{ \fL \norma{\xp - x_0} }\\
		&\qquad\overset{(i)}{\le} \Fo{\xopt} + \O{ \phi\frac{\fL\norma{\xopt - x_0}\sqrt{\ln\frac{\psi_\alpha}{\delta}}}{\sqrt{N}} } + \frac{\prn*{\fL \norma{\xp - x_0} + L\opt \norma{\xopt - x_0}}\ln(1/\delta)}{N}\\
		&\qquad\overset{(ii)}{\le} \Fo{\xopt} + \O{ \phi\frac{L\opt\norma{\xopt - x_0}\sqrt{\ln\frac{\psi_\alpha}{\delta}}}{\sqrt{N}} }, \numberthis
		\label{eq: fake F A 1}
	\end{align*}
	where $(i)$ is from \Cref{eq: min fake F 1,eq: F perfect bound}, and $(ii)$ is from $\norma{\xp - x_0}\le\hat{R}$, $\fL\le L\opt$, and from the fact that for every $v\in\brk*{0,1}$ we have $v\le\sqrt{v}$.

	Now, let us consider the case that $\hat{R} \ge 66\norma{\xopt - x_0}$.
	Thus, from \Cref{eq: min fake F 2,eq: F perfect bound} and from $\norma{\xp - x_0}\le\hat{R}$,
	\begin{align*}
		\Ft{\xp}{\fL\norma{\xp-x_0}}
		&\overset{(i)}{\le} \Fo{\xopt} - 2\frac{\fL\hat{R}\sqrt{\ln\frac{\psi_\alpha}{\delta}}}{\sqrt{N}} + \frac{\prn*{\fL \norma{\xp - x_0} + L\opt \norma{\xopt - x_0}}\ln(1/\delta)}{N}\\
		&\overset{(ii)}{\le} \Fo{\xopt} - 2\frac{\fL\hat{R}\sqrt{\ln\frac{\psi_\alpha}{\delta}}}{\sqrt{N}} + \prn*{\fL \norma{\xp - x_0} + L\opt \norma{\xopt - x_0}}\frac{\sqrt{\ln\frac{\psi_\alpha}{\delta}}}{\sqrt{N}}\\
		&\overset{(iii)}{\le} \Fo{\xopt} + \O{\frac{L\opt \norma{\xopt - x_0}\sqrt{\ln\frac{\psi_\alpha}{\delta}}}{\sqrt{N}}},
		\numberthis
		\label{eq: fake F A 2}
	\end{align*}
	where $(i)$ is from \Cref{eq: min fake F 2,eq: F perfect bound}, $(ii)$ is from the fact that for every $v\in\brk*{0,1}$ we have $v\le\sqrt{v}$, and $(iii)$ is from $\norma{\xp - x_0}\le\hat{R}$, and $\fL\le L\opt$.

	Therefore, \Cref{eq: fake F A 1,eq: fake F A 2} establish that in all cases, with probability at least $1-3\delta$, we have
	\begin{align*}
		\Ft{\xp}{\fL\norma{\xp-x_0}}
		&\le \Fo{\xopt} + \O{ \phi\frac{L\opt\norma{\xopt - x_0}\sqrt{\ln\frac{\psi_\alpha}{\delta}}}{\sqrt{N}} }.
	\end{align*}
	Now, \Cref{eq: fake L quantile} shows that $\fL$ exceeds the $1-\frac{\ln(1/\delta)}{N}$ quantile of $L\prn*{\PP}$.
	As
	\begin{align*}
		2 N + \frac{\ln\prn*{1/\delta} + 2e\ln(1/\delta) }{N - e \ln(1/\delta)}\cdot N
		\le 2 N + \frac{\ln\prn*{1/\delta} + 2e\ln(1/\delta) }{e \ln(1/\delta)}\cdot N
		\le 5N,
	\end{align*}
	\Cref{lem: redraw samples} shows that among $5N$ samples, with probability at least $1-\delta$, at least $2N$ lie below this quantile.
	Hence, with high probability, \Cref{alg: sample complexity clipping around initial} does not return FAILURE after the loop, and $6N$ samples suffice.

	Finally, by using the union bound, we get the desired results with probability at least $1-5\delta$.
\end{proof}

\section{Discussion of Lipschitz properties in the model--loss decomposition setting}

In this section, we discuss the relationship between the Lipschitz constants of the model $m$, the loss $h$, and the composed objective function $f$.

First, we note that the dual norm $\normas{ \cdot }$ can be viewed as the operator norm induced by the norms $\norma{\cdot}$ and $\abs{\cdot}$; both are defined as $\normas{ v } \defeq \sup\crl*{ \abs{\tri{v,x}} ~|~ \norma{x}\le1 }$.
Similarly, the dual norm $\normbs{ \cdot }$ can be viewed as the operator norm induced by the norms $\normb{\cdot}$ and $\abs{\cdot}$.
Therefore, immediately from the sub-multiplicativity of operator norms, we obtain that
\begin{align}
	\label{eq: grad f bound}
	\normas{\grad \ft{x}{s}}
	&= \normas{ \grad \lt{ \mt{x}{s} }{s} \grad \mt{x}{s} }
	\le \normbs{\grad \lt{ \mt{x}{s} }{s}} \cdot \normo{\grad \mt{x}{s}}
\end{align}
For completeness, we provide a short proof:
\begin{proof}
	For any $v\in\X$ such that $\norma{v}\le1$,
	if $\grad \mt{x}{s} v = 0$ then
	\begin{align*}
		\abs{ \grad \lt{ \mt{x}{s} }{s} \grad \mt{x}{s} v } = 0 \le \normbs{ \grad \lt{ \mt{x}{s} }{s} } \cdot \normo{\grad \mt{x}{s}}.
	\end{align*}
	Otherwise,
	\begin{align*}
		\abs{ \grad \lt{ \mt{x}{s} }{s} \grad \mt{x}{s} v }
		&=\abs{ \grad \lt{ \mt{x}{s} }{s} \frac{\grad \mt{x}{s} v}{\normb{\grad \mt{x}{s} v}} } \cdot \normb{\grad \mt{x}{s} v}\\
		&\le \normbs{ \grad \lt{ \mt{x}{s} }{s} } \cdot \normo{\grad \mt{x}{s}},
	\end{align*}
	where the inequality follows from the definitions of the dual norm $\normbs{\cdot}$ and the operator norm $\normo{\cdot}$.
	Thus, from the definition of the dual norm $\normas{\cdot}$, we obtain that \Cref{eq: grad f bound} holds.
\end{proof}

From \Cref{eq: grad f bound}, it immediately follows that if the model $m$ is $L$-Lipschitz and the loss function $h$ is $1$-Lipschitz, then the objective function $f$ is $L$-Lipschitz.
Formally:
\begin{lemma}
	\label{lem: f lip}
	Let $L>0$.
	\begin{enumerate}
		\item\label{lem: f lip 1}  If $m\in\ILip[0][L][\normo{\cdot}]$ and $h\in\ILip[0][1][\normbs{\cdot}]$, then $f\in\ILip[0][L][\normas{\cdot}]$.
		\item \label{lem: f lip 2} If $m\in\IsLip[0][L][\normo{\cdot}]$ and $h\in\ILip[0][1][\normbs{\cdot}]$, then $f\in\IsLip[0][L][\normas{\cdot}]$.
	\end{enumerate}
\end{lemma}

\begin{proof}
	\arxiv{~}
	\paragraph{Proof of \titleCrefi{lem: f lip}{lem: f lip 1}}
	For every $x\in\X$ and $s\in\S$, from \Cref{eq: grad f bound} we have
	\begin{align*}
		\normas{\grad \ft{x}{s}}
		\le \normbs{\grad \lt{ \mt{x}{s} }{s}} \cdot \normo{\grad \mt{x}{s}} \le L.
	\end{align*}
	Thus, $f\in\ILip[0][L][\normas{\cdot}]$.

	\arxiv{~}
	\paragraph{Proof of \titleCrefi{lem: f lip}{lem: f lip 2}}
	For every $x\in\X$, from \Cref{eq: grad f bound} we have
	\begin{align*}
		\E_{s\sim \PP} \normas{\grad \ft{x}{s}}^2
		\le \E_{s\sim \PP} \brk{ \normbs{\grad \lt{ \mt{x}{s} }{s}}^2 \cdot \normo{\grad \mt{x}{s}}^2}
		\le \E_{s\sim \PP} \brk{\normo{\grad \mt{x}{s}}^2 }
		\le L^2.
	\end{align*}
	Thus, $f\in\IsLip[0][L][\normas{\cdot}]$.
\end{proof}

However, it is possible to construct a model $m$ and a loss function $h$ such that the Lipschitz constant of $f$ is much smaller than the product of the Lipschitz constants of $m$ and $h$.
In what follows, we present examples in which the Lipschitz constant of $f$ is close to this product, showing that the bound above can be tight.

\subsection{Lipschitz constant under the absolute loss}

In this subsection, we discuss the Lipschitz constant of a model with the absolute loss.
Let $\Y\subseteq\R$, and let the loss function be the absolute loss: $\lt{y}{s}\defeq \abs{ y - q\prn*{ s } }$, where $q:\S\mapsto\R$ is some function.
It is easy to see that for every $y\in\Y$ and $s\in\S$ such that $y \neq q(s)$, we have $\grad \lt{y}{s} \in \{-1, 1\}$.
Therefore, for every $x\in\X$ and $s\in\S$ such that $\mt{x}{s} \neq q(s)$, we have $\normas{\grad \ft{x}{s}} = \normas{ \grad \lt{ \mt{x}{s} }{s} \grad \mt{x}{s} } = \normas{ \grad \mt{x}{s} }$.
Thus, for many models, the Lipschitz constant of $f$ coincides with that of $m$.

\newcommand{\mat}[1]{\text{mat}\prn*{ #1 }}
\subsection{Linear model with cross-entropy loss}

In this section, we discuss the Lipschitz constant of a linear model with cross-entropy loss.
Let $d_1,d_2\in\N$, $\X\subseteq\R^{d_1}$, and $\Y\subseteq\R^{d_2}$.
We consider the Euclidean norm $\norma{\cdot}=\normb{\cdot}=\norm{\cdot}_2$ on $\X$ and on $\Y$.
Define the softmax function by
\begin{align*}
	p_i\prn*{ y } \defeq \frac{e^{y_i}}{ \sum_{j=1}^{d_2} e^{y_j} }\qquad \boldsymbol{p}\prn*{ y } \defeq \prn*{p_1\prn*{ y }, \dots, p_{d_2}\prn*{ y }}.
\end{align*}

With a slight abuse of notation, let $s$ be a vector of size $d_1+1$, where $s_0$ denotes the class label and $s_{1:d_1}$ the remaining features.

\begin{proposition}
	Let the sample set $\S$ be a subset of $\crl{ s\in\R^{d_1+1} ~|~ s_0\in\brk*{d_2} }$.
	Let $m$ be the linear model $\mt{x}{s}\defeq \mat{x} s_{1:d_1}$, where $\mat{x}_{i,j} \defeq x_{j+(i-1)d_1}$ is the matricization of $x$.
	Define the cross-entropy loss $h$ as $\lt{y}{s}\defeq -\log\prn*{p_{s_0}\prn*{ y }}$.
	Consequently, the objective function $f$ is defined as $\ft{x}{s}\defeq -\log\prn*{p_{s_0}\prn*{ \mat{x} s_{1:d_1} }}$.

	\Cref{eq: grad f bound} is tight, i.e., for all $x\in\X$ and $s\in\S$,
	\begin{align*}
		\norma[2]{\grad \ft{x}{s}} = \norm{\grad \lt{\mt{x}{s}}{s}}_2 \cdot \normo[2][2]{ \grad \mt{x}{s} }.
	\end{align*}
	Furthermore, $h\in\ILip[0][\sqrt{2}][\normbs{\cdot}]$, and for $x=0$ and all $s\in\S$ we have
	$\norm{\grad \lt{\mt{0}{s}}{s}}_2 = \sqrt{1 - \frac{1}{d_2}}$.
\end{proposition}

\begin{proof}
	Formally, for every $i\in\brk*{d_2}$ and $j\in\brk*{d_1 d_2}$, if $j\in \brk{\prn*{i-1}\cdot d_1+1, i\cdot d_1}$ we have $\brk*{\grad \mt{x}{s}}_{i,j} = s_{j-\prn*{i-1} d_1}$, and otherwise $\brk*{\grad \mt{x}{s}}_{i,j} =0$.
	Informally, every row $i$ of $\grad \mt{x}{s}$ contains the sample $s_{1:d_1}$ starting at the $\prn*{i-1}\cdot d_1+1$ element,
	\setcounter{MaxMatrixCols}{15}
	\begin{align*}
		\grad \mt{x}{s} =
		\begin{bmatrix}
			s_1& \dots&s_{d_1} & 0 & \dots & 0 & 0& \dots&0 & 0 & \dots & 0 & 0& \dots&0\\
			0& \dots&0 & s_1 & \dots & s_{d_1} & 0& \dots&0 & 0  & \dots & 0 & 0& \dots&0 \\
			\vdots& & & & & & & \ddots & & & & & & & \vdots \\
			0& \dots&0 & 0 & \dots & 0 & 0& \dots&0 & s_1  & \dots & s_{d_1} & 0& \dots&0\\
			0& \dots&0 & 0 & \dots & 0 & 0& \dots&0 & 0& \dots&0 & s_1  & \dots & s_{d_1}
		\end{bmatrix}.
	\end{align*}
	Moreover, the gradient of the loss satisfies $\grad \lt{y}{s} = \boldsymbol{p}\prn*{ y } - \boldsymbol{e}_{s_0}$, where $\boldsymbol{e}_{s_0}$ is the vector with a one in the $s_0$-th position and zeros everywhere else.

	Since the operator norm $\normo[2][2]{\cdot}$ is the maximum singular value, we obtain $\normo[2][2]{ \grad \mt{x}{s} }=\norm{s_{1:d_1}}_2$.
	Therefore, by the chain rule,
	\begin{align*}
		\norma[2]{\grad \ft{x}{s}}
		&= \norma[2]{\grad \lt{\mt{x}{s}}{s}  \grad \mt{x}{s}}\\
		&= \sqrt{\sum_{i=1}^{d_2}\sum_{j=1}^{d_1} \brk*{\grad \lt{\mt{x}{s}}{s} }_i^2 s_j^2}\\
		&= \sqrt{\sum_{i=1}^{d_2} \brk*{\grad \lt{\mt{x}{s}}{s} }_i^2 \sum_{j=1}^{d_1} s_j^2}\\
		&= \norm{\grad \lt{\mt{x}{s}}{s}}_2 \cdot \norm{s_{1:d_1}}_2.
	\end{align*}
	Thus, for all $x\in\X$ and $s\in\S$ we have
	\begin{align*}
		\norma[2]{\grad \ft{x}{s}} = \norm{\grad \lt{\mt{x}{s}}{s}}_2 \cdot \normo[2][2]{ \grad \mt{x}{s} }.
	\end{align*}

	In addition, for all $x\in\X$ and $s\in\S$, we have
	\begin{align*}
		\norma[2]{\grad \lt{\mt{x}{s}}{s}}^2
		= 1 - 2p_{s_0}\prn*{ \mt{x}{s} } + \sum_{i=1}^{d_2} p_i^2\prn*{ \mt{x}{s} }
		\le 1 + \sum_{i=1}^{d_2} p_i\prn*{ \mt{x}{s} }
		= 2.
	\end{align*}
	Thus, $\norma[2]{\grad \lt{\mt{x}{s}}{s}} \le \sqrt{2}$.

	Moreover, for $x=0$ and all $s\in\S$ we have
	\begin{align*}
		\norma[2]{\grad \lt{\mt{0}{s}}{s}}^2
		= 1 - 2p_{s_0}\prn*{ \mt{0}{s} } + \sum_{i=1}^{d_2} p_i^2\prn*{ \mt{0}{s} }
		= 1 - \frac{2}{d_2} + \sum_{i=1}^{d_2} \frac{1}{d_2^2}.
	\end{align*}
	Thus, $\norm{\grad \lt{\mt{0}{s}}{s}}_2 = \sqrt{1 - \frac{1}{d_2}}$.
\end{proof}

This shows that, for linear models with cross-entropy loss, the upper bound in \Cref{eq: grad f bound} is tight, i.e., for any point $x\in\X$ and sample $s\in\S$ the gradient norm $\norma[2]{\grad \ft{x}{s}}$ is equal to $\norma[2]{\grad \lt{\mt{x}{s}}{s}} \cdot \normo[2][2]{ \grad \mt{x}{s} }$.
Therefore, depending on the choice of $\X$, $\Y$, and $\P$, the Lipschitz constant of the composed objective $f$ can be close to, or even equal to, the product of the Lipschitz constants of the model $m$ and the loss $h$.

\section{Extending the results to non-differentiable functions}
\label{sec: remove differentiable}

In this paper, we do not impose convexity requirements on the model and loss functions, but only on the objective function.
Therefore, subgradients do not necessarily exist at all points.
Instead, for a function $q$, we use the Clarke generalized gradient \citep{clarke1975generalized}:
\begin{align*}
	\clarke q\prn*{x}\defeq\operatorname{co}\crl*{ v ~|~ \exists \crl{x_i}_{i\in\N} \subset \Omega_q : x = \lim_{i\rightarrow\infty} x_i \text{ and } v = \lim_{i\rightarrow\infty}\nabla q\prn*{x_i} },
\end{align*}
where $\operatorname{co}$ is the convex hull, and $\Omega_q$ is the set of points where $q$ is differentiable.
We note that if $q$ is convex and locally Lipschitz, then $\clarke q(x)$ coincides with the set of subgradients $\partial q(x)$ for every $x$ in any open subset of the domain \citep[Proposition 1.2]{clarke1975generalized}; see also \citet[Theorem 25.6]{rockafellar1970a}.

We define $\Icont$ as the class of functions $q$ such that:
(i) $\qt{\cdot}{s}$ is differentiable almost everywhere for any sample $s$;
(ii) the Clarke gradient $\clarke \qt{x}{s}$ is nonempty for every point $x$ and sample $s$;
and (iii) for every pair of points $x,y$, sample $s$, and every continuously differentiable curve $\gamma:[0,1]\to\operatorname{dom}\prn*{\qt{\cdot}{s}}$ satisfying $\gamma(0)=y$ and $\gamma(1)=x$, there exists a selection $\grad_x \qt{\gamma\prn*{a}}{s}\in\clarke\qt{\gamma\prn*{a}}{s}$ such that
\begin{align*}
	\qt{x}{s} = \qt{y}{s} + \int_0^1 \grad_x \qt{\gamma\prn*{a}}{s} \grad_a\gamma\prn*{a} \,da.
\end{align*}

We further assume that the model--loss decomposition satisfies the chain rule; i.e., for all $x\in\X$ and $s\in\S$,
\begin{align*}
	\clarke \ft{x}{s} \subseteq \overline{\operatorname{co}} \crl*{ v u ~|~ v\in\clarke\lt{ \mt{x}{s} }{s} \text{ and } u\in\clarke\mt{x}{s} },
\end{align*}
where $\overline{\operatorname{co}}$ denotes the closed convex hull.
The purpose of this requirement is to ensure that if $m$ is $L$-Lipschitz and $h$ is $1$-Lipschitz, then $f$ is also $L$-Lipschitz, as in the differentiable case (see \Cref{lem: f lip}).

We are now able to replace the use of $\Idiff$ with $\Icont$.
For the classes of convex stochastic optimization problems, the definitions remain almost the same:
\begin{align*}
	\IDist &\defeq \crl*{q\in\Icont ~|~ \norma{x_0 - \xopt} \le \urho \text{ and } \E_{s\sim \PP} \qt{\cdot}{s} \text{ is convex}},\\
	\IDistf &\defeq \crl*{q\in\Icont ~|~ \norma{x_0 - \xopt} \le \urho \text{ and } \forall s\in \S: \qt{\cdot}{s} \text{ is convex}}.
\end{align*}
To account for the entire Clarke generalized gradient, we modify the definitions of the classes of Lipschitz and second-moment-Lipschitz problems:
\begin{align*}
	\ILip &\defeq \crl*{q\in\Icont ~|~ \forall x\in\X; s\in \S; g\in\clarke \qt{x}{s}: \norm{g} \le \uell },\\
	\IsLip &\defeq \crl*{q\in\Icont ~|~ \forall x\in\X: \E_{s\sim \PP} \brk*{\sup_{g\in\clarke \qt{x}{s}} \norm{g}^2} \le \uell^2 }.
\end{align*}

Except for the proof of \Cref{thm: sample algorithm}, all other proofs remain valid under the modified definitions of the problem classes, provided that $\nabla q(x)$ is interpreted as an arbitrary element of $\clarke q(x)$ at the point $x$.
\Cref{thm: sample algorithm}, however, requires assuming that $f$ is differentiable, because its proof uses Lemma 2 of \citet{lawrence2025sample}, which relies on this differentiability assumption.
\Cref{thm: sample algorithm} can be extended to the non-differentiable setting via smoothing arguments, e.g., using the Moreau envelope \citep{moreau1965proximity}.

\end{document}